\documentclass[twoside]{article}

\usepackage[accepted]{aistats2022}
\usepackage[round]{natbib}

\usepackage{wustyle}

\begin{document}

%

%

\twocolumn[

\aistatstitle{Learning a Single Neuron for Non-monotonic Activation Functions}

\aistatsauthor{Lei Wu}

\aistatsaddress{School of Mathematical Sciences, Peking University \\ leiwu@math.pku.edu.cn}
]

\begin{abstract}
We study the problem of learning a single neuron $\bx\mapsto \sigma(\bw^T\bx)$ with gradient descent (GD). All the existing positive results are limited to the case where $\sigma$ is monotonic. However, it is recently observed that non-monotonic activation functions outperform the traditional monotonic ones in many applications. To fill this gap, we establish learnability without assuming monotonicity. Specifically, when the input distribution is the standard Gaussian, we show that mild conditions on $\sigma$ (e.g., $\sigma$ has a dominating linear part) are sufficient to guarantee the learnability in polynomial time and polynomial samples. Moreover, with a stronger assumption on the activation function, the condition of input distribution can be relaxed to a non-degeneracy of the marginal distribution. We remark that our conditions on $\sigma$  are satisfied by practical non-monotonic activation functions, such as SiLU/Swish and GELU. We also discuss how our positive results are related to existing negative results on training two-layer neural networks.
\end{abstract}

\section{Introduction}
Neural networks play a fundamental role in deep learning, which has achieved unprecedented successes in many applications, such as computer vision, natural language processing, and scientific computing. Despite tremendous efforts devoted, theoretical understandings of learning neural networks are still rather unsatisfactory because of the inherent non-convexity. 

In this paper, we consider the simplest setting: learning a single neuron $\bx\mapsto \sigma(\bw^T\bx)$, where $\bw$ is the parameter to be learned and $\sigma:\RR\mapsto\RR$ is a fixed activation function.  This problem has been widely studied previously (see the related work section below for more details) and plays an important role in understanding general neural networks, e.g., the superiority of neural networks over kernel methods \citep{yehudai2019power} and the hardness of training neural networks \citep{shamir2018distribution,livni2014computational}.

 We assume that inputs are drawn from an underlying distribution $\cD$ and the labels are generated by some unknown neuron $\bx\mapsto\sigma({\bw^*}^T\bx)$, i.e., the realizable case. As such, the population risk is given by 
\[
    \cR(\bw):=\EE_{\bx}[\frac{1}{2}(\sigma(\bw^T\bx)-\sigma({\bw^*}^T\bx))^2].
\]
In practice, only finite training samples $\{\bx_i\}_{i=1}^n$ are available, and we instead minimize the empirical risk
\[
    \hat{\cR}_n(\bw):=\frac{1}{2n}\sum_{i=1}^n(\sigma(\bw^T\bx_i)-\sigma({\bw^*}^T\bx_i))^2.
\]
Despite the simplicity, this problem is still highly non-trivial due to the non-convexity. 

To attack this problem, existing works \citep{frei2020agnostic,mei2018landscape,yehudai2020learning,tian2017analytical} all assume $\sigma$ to be monotonic,  for which
\begin{align}\label{eqn: monotonic}
 \notag   &G(\bw)=\langle \nabla \cR(\bw), \bw-\bw^*\rangle \\
\notag &= \EE_{\bx}[(\sigma(\bw^T\bx)-\sigma({\bw^*}^T\bx))\sigma'(\bw^T\bx)(\bw^T\bx-{\bw^*}^T\bx)]\\
&\geq 0.
\end{align}
The above inequality implies two critical facts:
\begin{itemize}
\item All local minima are global minima if $\sigma$ is monotonic. Note that \eqref{eqn: monotonic} implies that  $\mathrm{d}\cR(\bw^*+\beta(\bw-\bw^*))/\mathrm{d} \beta \geq 0$ for any $\bw\in\RR^d$. This suggests that starting from any $\bw$, we can find a loss-decreasing curve connecting $\bw$ to $\bw^*$. Therefore, there is no bad local minima; see also \cite[Theorem 5.1]{auer1996exponentially}.
\item The gradient at every point points to a direction of decreasing $\|\bw_t-\bw^*\|$ since $\mathrm{d}\|\bw_t-\bw^*\|^2/\mathrm{d}t=-G(\bw_t)\leq 0$. 
\end{itemize}
Moreover, \cite{frei2020agnostic,mei2018landscape,yehudai2020learning,tian2017analytical} impose stronger assumptions on the activation function and input distribution to ensure the lower boundedness of $G(\bw)$, thereby guaranteeing the convergence. 





However, recent practical evidence \citep{devlin2018bert,radford2018improving,radford2019language,sitzmann2020implicit} shows that in many applications,  non-monotonic activation functions, e.g., SiLU and GELU,  are superior to the traditional monotonic ones (See the related work section below for more details). This motivates us to analyze the case where $\sigma$ is non-monotonic. Note that for general activation functions, the risk landscape may have a large number of bad local minima  \citep{brady1989back,ros2019complex}. 
Moreover, \cite{shamir2018distribution,Livni} show that if $\sigma$ is highly oscillated, gradient-based  methods suffer from the curse of dimensionality in learning a single neuron. These suggest that some conditions (beyond the monotonicity) on $\sigma$ must be imposed to ensure learnability. 

Our \textbf{main contributions} are be summarized as follows.
\begin{itemize}
\item  We first consider activation functions that are increasing in $[0,\infty)$  and satisfy $\inf_{0<z<\alpha} \sigma'(z)\geq \gamma, \inf_{z_1\geq 0, z_2\leq 0}\sigma'(z_1)\sigma'(z_2)\geq - \zeta^2$ for some constants $\alpha, \gamma, \zeta>0$. We prove in Theorem \ref{pro: GD-convergence-nonsymmetric} that if the input distribution $\cD$ is sufficiently ``spread'', GD converges to a global minimum exponentially fast as long as $\gamma$ is relatively larger than $\zeta$. This condition essentially means that the monotonic component of the activation function dominates. 

\item  Then fine-grained analyses of the GD  dynamics are provided for the case where the input distribution is the standard Gaussian. In this case, the condition on $\sigma$ can be further relaxed. Specifically, we consider two settings: GD with zero initialization and Riemannian GD with a random initialization.

The analysis of zero initialization relies on the observation that the gradient at zero points to the ground truth $\bw^*$ and therefore, the original problem can be reduced to minimizing a one-dimensional risk.  The same observation has been exploited in \cite{tian2017analytical,soltanolkotabi2017learning,kalan2019fitting} for the specific ReLU activation function, whereas we show that it holds for general activation functions. In addition,  we identify further conditions on $\sigma$ to ensure that this one-dimensional risk has a benign landscape,  thereby guaranteeing the convergence of GD.

For random initialization, we consider the Riemannian GD with $\bw_t\in\SS^{d-1}$. For this case, we show that the population risk has  a simple closed-form analytic expression (see Lemma \ref{lemma: risk-expression}), which depends on $\sigma$ only through the Hermite coefficients $\{\hat{\sigma}_k\}_k$. Here $\hat{\sigma}_k = \EE_{z\sim N(0,1)}[h_k(z)\sigma(z)]$, where $h_k$ is the $k$-th probabilistic Hermite polynomial. By using this analytic expression, we provide a thorough study of how the decay of Hermite coefficients affects the property of risk landscape and the convergence of Riemannian GD. In particular, we establish in Proposition \ref{pro: high-prob} a high-probability convergence to the global minimum by assuming that the linear component of the activation function, i.e., $\hat{\sigma}_1=\EE[z\sigma(z)]$, is sufficiently large. 
On the other hand, if $\hat{\sigma}_1=0$, we construct a counterexample in Lemma \ref{lemma: random-init-optimality} , for which the Riemannian GD converges to a bad local minimum with a probability close to $1/2$. These together partially explain the wide use of ReLU and its variants in practice since they all have dominating linear components.

\item Lastly, we consider the finite sample case.  In Proposition \ref{lemma: emp-pop-gap}, we establish  the closeness between the empirical landscape and the population landscape using the theory of empirical process. With these closeness results, we can convert our positive results of the population GD to the empirical GD (see Proposition \ref{pro: empirical-GD-zero-init} and Proposition \ref{pro: empirical-GD-sphere} ). In particular, in all the settings, we show that GD can learn the ground truth using only polynomial samples and polynomial time.
\end{itemize}

Note that, for all the settings we considered, the conditions are satisfied by all the popular activation functions used in practice, including the non-monotonic ones. 

\subsection{Related work}
\vspace*{-.5em}
\label{sec: related-work}

\paragraph*{Non-monotonic activation functions}
 \cite{ramachandran2017searching} uses  the neural architecture search (NAS) method to search the best activation function for classifying the CIFAR-10 data. It is discovered that the non-monotonic Swish function, $\sigma_{\text{swish}}(z)=z\sigma_{\text{sigmoid}}(\beta z)$ with $\beta>0$, performs the best.   In particular, when $\beta=1$, it becomes the sigmoid-weighted linear unit (SiLU) \citep{elfwing2018sigmoid}. Recently, SiLU/Swish also show extraordinary performances on many other applications, such as adversarial training \citep{xie2020smooth}, model compression \citep{tessera2021keep}, etc. 
Gaussian error linear unit (GELU) \citep{hendrycks2016gaussian} is another popular non-monotonic activation function and has the similar properties to  SiLU/Swish. GELU has been widely applied in large-scaled pre-trained language models, such as  GPT/GPT-2 \citep{radford2018improving,radford2019language}, BERT \citep{devlin2018bert}, and most other Transformer-based \citep{vaswani2017attention} models \citep{liu2019roberta}. 
In addition, non-monotonic activations also see lots of applications in solving  scientific computing problems. For these problems, one may need to restrict the activation function to be periodic or  compactly supported, where activation functions are always  non-monotonic, e.g., \cite{sitzmann2020implicit,li2020multi,liang2021reproducing,chen2020comparison} to name a few. Therefore, understanding the learning of neural networks with non-monotonic activation functions becomes crucially important.

\paragraph*{Learning a single neuron under the realizable setting}
A single neuron is essentially the same as the traditional generalized linear models (GLMs) and single-index models (SIMs). For GLMs, $\sigma$ is usually a nondecreasing function, such as the sigmoid function for the logistic binary classification. Except for the monotonicity, SIMs further assume that $\sigma$ is unknown, to be learned from data. When $\sigma$ is nondecreasing,  $\sigma^{-1}(\cdot)$ can be defined. Hence, this problem can be efficiently solved by fitting the linear function: $\bx\mapsto \sigma^{-1}(y)$. Indeed, the algorithms for GLMs and SIMs are based on this observation \citep{kalai2009isotron,kakade2011efficient}, which obviously does not hold if $\sigma$ is non-monotonic. 

In contrast, the GD method is applicable irrespective of the monotonicity of $\sigma$. However, the theoretical understanding of GD is non-trivial because of the non-convexity of the risk landscape.  
When $\sigma$ is strictly monotonic and $\cD$ is non-degenerate, there is only one critical point: $\bw=\bw^*$ demonstrated previously. However, 
for general activation functions and general input distribution, there may exist many bad local minima and saddle points \citep{brady1989back,ros2019complex}. Moreover, the empirical landscape  can be much more complex. For instance, even when $\sigma$ is strictly monotonic, there may exist many bad critical points when $n/d\leq c_\sigma$ for some constant $c_\sigma>0$. Using Kac-Rice replicated method \citep{ros2019complex} from theoretical physics, \cite{maillard20a} provides an explicit characterization of the critical points in the thermodynamics limit: $n,d\to\infty$ with $n/d\to \alpha>1$. Lastly we mention that for the non-realizable case, there may exist  bad local minima even if $\sigma$ is strictly monotonic \citep{auer1996exponentially}.

Apart from the above landscape analyses, the hardness of learning can be substantiated for the case where $\sigma$ is periodic. Specifically, \cite{kearns1998efficient,blum1994weakly,diakonikolas2020algorithms,malach2020hardness} shows that learning the parity function: $\{0,1\}^d\mapsto \{-1,1\}: f_{\bv}(x)=(-1)^{\bv^Tx}$ suffers from the curse of dimensionality. The parity function is essentially a single neuron with $\sigma(z)=(-1)^z$ and $\cD=\text{Unif}(\{0,1\}^d)$.
\cite{shamir2018distribution} later extends the above understanding to general periodic activation functions and $\cD=\cN(0,I_d)$. Our results are consistent with these negative results since the constants in our bounds are exponentially large for these periodic activations. Moreover, our results imply that GD can learn a single neuron efficiently as long as the activation function does not oscillate too much.


The previous positive results are summarized as follows. \cite{mei2018landscape,oymak2019overparameterized} show  that the empirical GD can return a good approximation of $\bw^*$. However, the analysis requires $\sigma$ to be \emph{strictly} monotonic.  \cite{yehudai2020learning} later shows that as long as the input distribution is sufficiently ``spread'',  a weak monotonicity condition  on $\sigma$ is sufficient to guarantee a constant-probability convergence for a random initialization. A similar analysis for the agnostic setting is provided in \cite{frei2020agnostic}.
For the specific ReLU activation function and standard Gaussian input distribution, \cite{tian2017analytical} proves the exponential convergence of the population GD. \cite{
soltanolkotabi2017learning,kalan2019fitting} considered a similar setting but for the empirical GD. Our work differentiates from these works by removing the requirement of monotonicity.



Another line of related research is phase retrieval \citep{sun2018geometric,tan2019online,chen2019gradient}, which fits our setting with $\sigma(z)=z^2$ or $\sigma(z)=|z|$. In phase retrieval, the activation function is indeed non-monotonic, but the analysis is specific to those activation functions. By contrast, our analysis holds for more general non-monotonic activation functions, including the popular SiLU/Swish and GELU.

\section{Preliminaries}
\vspace*{-.5em}
\paragraph*{Notation}
Let $I_d$ denote the $d\times d$ identity matrix. 
We use bold-faced letters to denote vectors.  For a vector $\bw$, let $w_i$ denote the $i$-th coordinate, $\|\bw\|^2=\sum_i w_i^2$. For $\bw,\bv\in\RR^d$, we use $\theta(\bw,\bv)=\arccos(\frac{\bw^T\bv}{\|\bw\|\|\bv\|})$ to denote the angle between $\bw$ and $\bv$. Let $\SS^{d-1}=\{\bw\in\RR^{d}:\|\bw\|=1\}$.
We use $X\lesssim Y$ to denote  $X\leq CY$ for some absolute constant $C>0$. We will occasionally use $\tilde{O}(\cdot)$ to hide  logarithmic factors.

For simplicity, we  assume that $\|\bw^*\|=1$ and $\sigma(0)=0$, otherwise, we can replace $\sigma(z)$ with $\sigma(z/\|\bw^*\|)-\sigma(0)$ without changing the risk landscape. 
The gradient of population risk can be written as
\begin{equation}\label{eqn: 1}
\nabla \cR(\bw) = \EE_{\bx}\big[(\sigma(\bw^T\bx)-\sigma({\bw^*}^T\bx))\sigma'(\bw^T\bx)\bx\big].
\end{equation}
When $\bw\neq 0$, as long as the marginal distribution $\bw^T\bx$ is not singular, \eqref{eqn: 1}  holds if $\sigma$ is differentiable almost everywhere, since changing the value of $\sigma'(z)$ at a set of measure zero does not affect the expectation. When $\bw=0$ and $\sigma(\cdot)$  is not differentiable at the origin, we will explicitly specify the value of $\sigma'(0)$, e.g., $\sigma'(0)=1$ for ReLU.

For the training method, we focus on the GD flow $\dot{\bw}_t = -\nabla \cR(\bw_t)$, which is  GD with an infinitesimal learning rate. Extending the results of GD flow to standard GD and stochastic gradient descent  for learning a single neuron is straightforward; we refer to \cite{yehudai2020learning} for some examples. Throughout this paper, we will use GD to denote GD flow for simplicity.

For non-monotonic activation functions, we are particularly interested in the \emph{self-gated family}: 
\begin{equation}\label{def: self-gated}
\sigma_{\beta}(z)=z\phi(\beta z),
\end{equation} 
where $\phi:\RR\mapsto\RR$ is nondecreasing and satisfies that $\phi(-\infty)=0, \phi(+\infty)=1$. As $\beta\to\infty$, $\sigma_{\beta}$ converges to ReLU. SiLU/Swish corresponds to the case that $\sigma$ is the sigmoid function. GELU corresponds to the case where $\phi$ is the cumulative density function of $\cN(0,1)$

\section{A General Result}
\vspace*{-.5em}
\label{sec: non-symmetric}

In this section, we make the following assumption.
\begin{assumption}\label{assumption: non-symmetric}
The following holds for some fixed $\alpha,\beta,\gamma,\zeta, \tau>0$:
\begin{itemize}
\item \textbf{Input distribution:} (1) $\EE_{\bx\sim\cD}[\bx\bx^T]\leq \tau I_d$. (2) For any $\bw\neq \bv\in \SS^{d-1}$, let $\cD_{\bw,\bv}$ denote the marginal distribution of $\bx$ on $\text{span}\{\bw,\bv\}$ (as a distribution over $\RR^2$). Let $p_{\bw,\bv}$ denote the density function of $\cD_{\bw,\bv}$. Assume $\inf_{\bz\in\RR^2:\|\bz\|\leq \alpha} p_{\bw,\bv}(\bz)\geq \beta$.
\item \textbf{Activation:}
$\sigma$ is increasing in $[0,\infty)$ and  $\inf_{z_1\geq 0, z_2\leq 0}\sigma'(z_1)\sigma'(z_2)\geq-\zeta^2, \sup_{0<z<\alpha}\sigma'(z)\geq \gamma$.
\end{itemize}
\end{assumption}
This assumption  is a modification of  \citep[Assumption 4.1]{yehudai2020learning}. The difference is that (1) $\sigma$ is allowed to be non-monotonic in $(-\infty, 0]$ and we further assume the second-order moment of $\cD$ to be bounded. The assumption on activation functions covers the popular self-gated family and excludes the hard examples where the activation function is periodic.
The assumption on $\cD$ is quite general and covers, for instance, log-concave distributions like Gaussian and uniform distributions with $\alpha,\beta,\tau=O(1)$.

\begin{proposition}\label{pro: non-symmetric-grad}
Let $\theta(\bw,\bw^*)$ be the angle between $\bw$ and $\bw^*$. 
For any $\delta\in (0,\pi)$, let $c_\delta = \sin^3(\delta/4)/(8\sqrt{2})$. 
Under Assumption \ref{assumption: non-symmetric}, for any $\bw\in\RR^{d}$ that satisfies $\theta(\bw,\bw^*)\leq \pi-\delta$, it holds that $ \lag \nabla \cR(\bw),\bw-\bw^*\rag \geq \lambda \|\bw-\bw^*\|^2$, where 
\[
   \lambda =(\gamma^2+\zeta^2)\beta\alpha^4 c_\delta -  \tau \zeta^2.
\]
\end{proposition}
This proposition implies that the gradient  $\nabla \cR(\bw)$ provides a good direction  for convergence as long as $\theta(\bw,\bw^*)$ is relatively small. In particular, $\lambda>0$ for any $\delta>0$ if $\zeta = 0$, and this corresponds to the monotonic case. In general, if $\gamma^2\beta\alpha^4c_\delta \geq  \tau\zeta^2$, we have 
$\lambda\geq \zeta^2\beta\alpha^4 c_\delta$. This condition means that the monotonic part of $\sigma$ dominates the non-monotonic part in the sense that $\frac{\gamma^2}{\zeta^2}\geq \frac{\tau}{c_\delta \beta\alpha^4}$. 
When $\cD=\cN(0,I_d)$, it is easy to verify that this condition is satisfied by the popular SiLU/Swish and GELU activations. The proof of Proposition \ref{pro: non-symmetric-grad} is presented in Appendix \ref{sec: app-non-sym}, which is modified from the proof of \cite[Theorem 4.2]{yehudai2020learning}.



\subsection{Convergence}
\vspace*{-.5em}
In this section, let $\delta_t=\pi - \theta(\bw_t,\bw^*)$. We explicitly write $\lambda(\delta_t)=\lambda$ to emphasize the dependence on the angle $\theta(\bw_t,\bw^*)$. Then,
Proposition \ref{pro: non-symmetric-grad} implies that $d\|\bw_t-\bw^*\|^2/dt\leq -\lambda(\delta_t) \|\bw_t-\bw^*\|^2\leq 0$. By the definition in Proposition \ref{pro: non-symmetric-grad}, we have $\lambda(\delta_t)\leq 0$ when $\delta_t=0$.  Therefore, for guaranteeing the convergence, we need to ensure that $\bw_t$ always stay in a region where $\delta_t=\pi-\theta(\bw_t,\bw^*)$ is significantly large. 

\paragraph*{Intuition.} The decreasing of $\|\bw_t-\bw^*\|$ does not alway imply the decreasing of $\theta(\bw_t,\bw^*)$. \cite{yehudai2020learning} shows that  $\theta(\bw_t,\bw^*)$ may increase and consequently $\lambda(\delta_t)$ decreases during the training.  Let $H_{+}=\{\bw\in \RR^d \,:\, \bw^T\bw^*\geq 0\}$. Obviously, $\theta(\bw,\bw^*)\leq \pi/2$ for any $\bw\in H_{+}$. The following lemma formalizes the preceding intuition.
\begin{lemma}\label{lemma: general-small-than-1}
If $\|\bw-\bw^*\|< 1$, then $\theta(\bw,\bw^*) < \frac{\pi}{2}$.
\end{lemma}
\begin{proof}
$\|\bw-\bw^*\|^2 = 1 - 2 \bw^T\bw^* + \|\bw\|^2< 1$ implies that $\bw^T\bw^*\geq \|\bw\|^2>0$. Hence, $\theta(\bw,\bw^*)<\frac{\pi}{2}$.
\end{proof}
Hence, if $\|\bw_0-\bw^*\|<1$,  the decreasing of $\|\bw_t-\bw^*\|$ can ensure that  $\|\bw_t-\bw*\|<1$ for all $t\geq 0$. Consequently,
$\delta_t=\pi-\theta(\bw_t,\bw^*)>\frac{\pi}{2}$ and $\lambda(\delta_t)> \lambda(\frac{\pi}{2})$ for any $t\geq 0$.  

\begin{theorem}\label{pro: GD-convergence-nonsymmetric}
Suppose that Assumption \ref{assumption: non-symmetric} holds and $\lambda(\frac\pi 2)>0$.
consider the random initialization $\bw_0\sim \cN(0, \eta^2 I_d)$ with $\eta\leq \frac{1}{\sqrt{2}d}$. Then, with probability at least $\frac{1}{2}-\frac{1}{4}\eta d-1.2^{-d}$ we have $\|\bw_0-\bw^*\|\leq 1-2\eta^2 d$ and 
$$
\|\bw_t-\bw^*\|^2\leq e^{-\lambda(\frac{\pi}{2})t}.
$$
\end{theorem}
This theorem provides a constant probability convergence.
Note that $\lambda(\frac\pi 2)=(\gamma^2+\zeta^2)\beta\alpha^4 c_{\frac \pi 2} -  \tau \zeta^2>0$ means that $\sigma$ has a dominated monotonic part. The specific choice of the variance of the random initialization can guarantee that $\|\bw_0-\bw^*\|<1$ holds with a constant probability (close to $1/2$). The proof is presented in  Appendix \ref{sec: app-non-sym}.

\section{Fine-Grained Analysis for Gaussian Inputs}
\vspace*{-2mm}
\label{sec: symmetric}
In this section, we provide a fine-grained analysis of the risk landscape and the convergence of GD for the case of $\cD=\cN(0,I_d)$. The main message is that  the conditions on $\sigma(\cdot)$ can be further relaxed. Similar results can be straightforward extended to other spherically symmetric distribution, e.g., $\text{Unif}(\SS^{d-1})$. 

\subsection{Zero Initialization}
\label{subsec: zero-init}
We first study GD with zero initialization.
The analysis mainly relies on the following observation.
\begin{lemma}\label{pro: gradient-gaussian}
$
\nabla \cR(\beta \bw^*) = - r'_\sigma(\beta)\bw^*,
$
where $r'_\sigma$ is the derivative of $r_\sigma:\RR\to\RR$ given by
\[
    r_\sigma(\beta) = \frac{1}{2}\EE_{z\sim\cN(0,1)}[(\sigma(\beta z)-\sigma(z))^2].
\]
\end{lemma}
\begin{proof}
Let $V=(\bv_1,\bv_2,\dots,\bv_d)^T\in\RR^{d\times d}$ be an orthonormal matrix with $\bv_1=\bw^*$ . Let $\tilde{\bx}=V\bx$. $\tilde{\bx}\sim\cN(0,I_d)$ and $\bx=V^T\tilde{\bx}=\sum_{j=1}^d \tilde{x}_j \bv_j$. Then, 
\begin{align*}
    \nabla \cR(\beta \bw^*) &= \EE_{\bx}[(\sigma(\beta {\bw^*}^T\bx)-\sigma({\bw^*}^T\bx))\sigma'(\beta{\bw^*}^T\bx)\bx]\\
    &= \EE_{\tilde{\bx}}[(\sigma(\beta \tilde{x}_1)-\sigma(\tilde{x}_1))\sigma'(\beta\tilde{x}_1)\sum_{j=1}^d \bv_j \tilde{x}_j]\\
    &= \EE_{\tilde{x}_1}[(\sigma(\beta \tilde{x}_1)-\sigma(\tilde{x}_1))\sigma'(\beta\tilde{x}_1)\tilde{x}_1] \bv_1 \\
    &:= - r_{\sigma}'(\beta) \bw^*,
\end{align*}
where the third equality is due to $\EE[h(\tilde{x}_1)\tilde{x}_j]=0$ for any $j\neq 1$.
\end{proof}
This lemma implies that $\nabla \cR(\bw)$ at the line $\{\bw=\beta \bw^*\,:\,\beta\in\RR\}$ exactly points to $\bw^*$ (maybe up to a sign). Therefore, GD starting zero will always stay on this line. Note that \citep{tian2017analytical,soltanolkotabi2017learning,kalan2019fitting} have made the same observation but only for the specific ReLU activation.

\begin{proposition}
Denote by $\bw_t$ the GD solution that starts from $\bw_0=0$. Then,
 $\bw_t=\beta_{t} \bw^*$ and $\beta_t$ is the GD solution that minimizes $r_\sigma(\cdot)$, i.e.,  
$
    \dot{\beta}_{t} = -r'_\sigma(\beta_t)
$
with $\beta_0=0$.
\end{proposition}
The proof is a straightforward application of Lemma \ref{pro: gradient-gaussian}.
It is implied that the  GD starting from $0$ is equivalent to an one-dimensional GD that minimizes $r_\sigma(\cdot)$. In particular, $\beta=1$ corresponds to the true solution. As a result, to ensure the convergence of GD, we only need  $r_\sigma(\cdot)$ to have a nice landscape in $[0,1+\delta]$ for some $\delta>0$.  Shown in 
Figure \ref{fig: 1dp-activation} are the landscapes of $r_\sigma(\cdot)$ for various commonly-used activation functions. On can see that for all the cases, $r_\sigma(\cdot)$ is monotonically decreasing in $[0,1]$, which implies that GD can converge to the global minimum $\beta=1$. Taking ReLU as a concrete example, we have 
\begin{align*}
r_{\sigma}(\beta)&=\frac{1}{2}\EE_{z\sim\mathcal{N}(0,1)}[|\sigma(\beta z)-\sigma(z)|^2]\\
&=\frac{(\beta-1)^2}{2}\EE_{z\sim\mathcal{N}(0,1)}[\sigma(z)^2]=\frac{(\beta-1)^2}{4}.
\end{align*}
This implies that $\beta_t$ converges exponentially fast. The following theorem generalizes it to  general activation functions.
\begin{figure}[!h]
\centering
\includegraphics[width=0.5\textwidth]{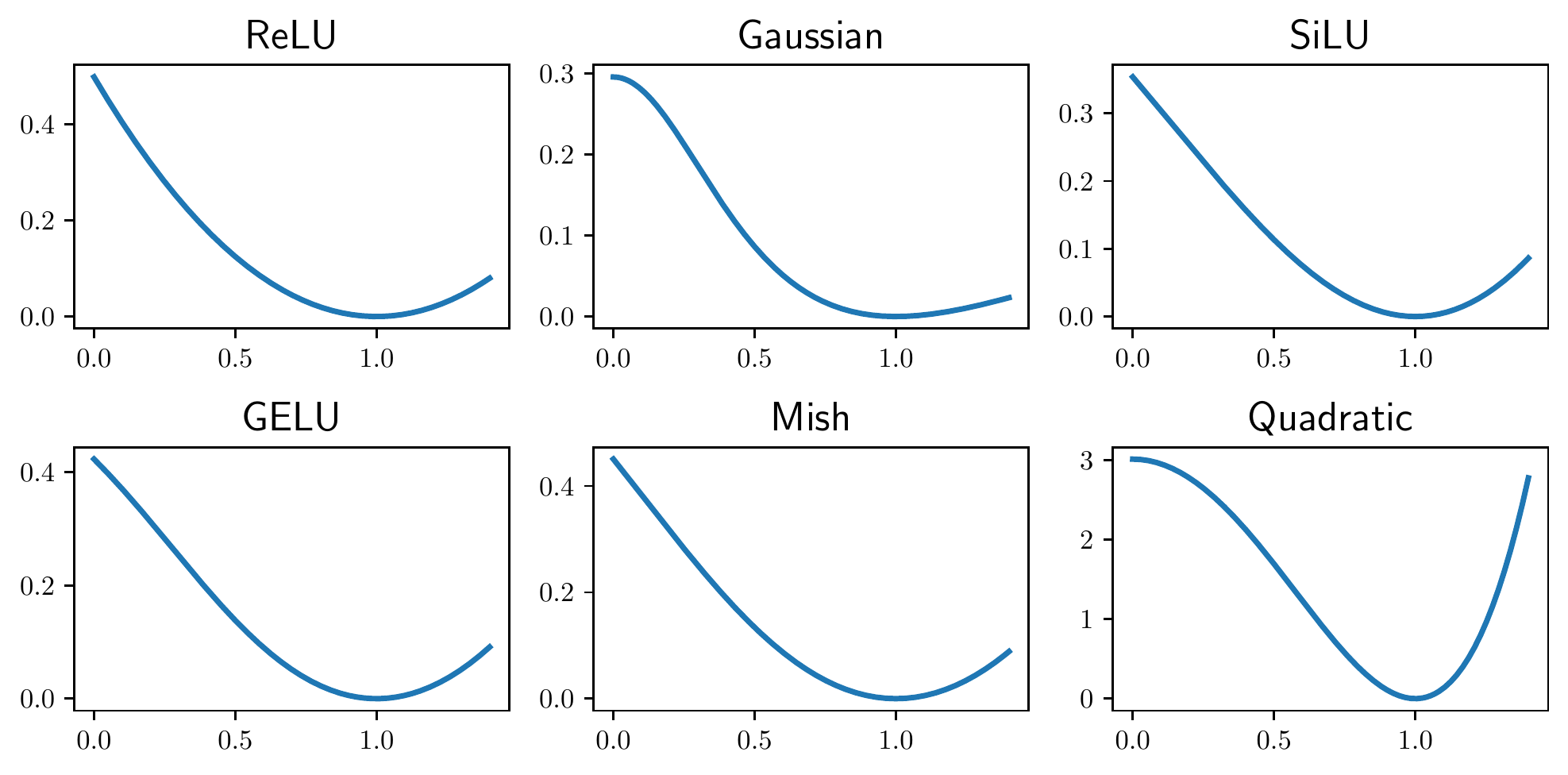}
\vspace*{-5mm}
\caption{\small The landscape of $r_\sigma(\cdot)$ for various activation functions.}
\label{fig: 1dp-activation}
\end{figure}

\begin{theorem}\label{thm: pop-GD}
Suppose that  $\sigma(\cdot)$  satisfies $r'_\sigma(\beta)\leq - C(1-\beta)$ for any $\beta\in [0,1]$ and some constant $C>0$. We have  $\|\bw_t-\bw^*\|\leq e^{-Ct}$.
\end{theorem}
\begin{proof}
 It is obvious  that 
$
\|\bw_t-\bw^*\|=1-\beta_t.
$
$\dot{\beta}_t = - r'(\beta_t)\geq C(1-\beta_t)$, which leads to $1-\beta_t\leq e^{-Ct}$. Hence, we complete the proof. 
\end{proof}
The assumption of the activation function in Theorem \ref{thm: pop-GD} is quite general but abstract. In the following, we substantiate it with some explicit assumptions.

\subsubsection{Monotonic activations}
\begin{lemma}\label{lemma: monotonic-pop-1d}
If $\sigma$ is monotonic, $r_\sigma(\cdot)$ is also monotonic in $[0,1]$. Furthermore, if there exists an interval $I=[z_0,z_1]$ such that $0\in I$ and $\sigma'(z)\geq C_1>0$ for $z\in I$. Then, there exists $C_2>0$ such that $r'_\sigma(\beta)\leq - C_2(1-\beta)$ for any $\beta \in [0,1]$.
\end{lemma}
\begin{proof}

If $\sigma$ is monotonically increasing, then $\sigma'(z)\geq 0$ a.e., thereby 
$(\sigma(z)-\sigma(\beta z))\sigma'(\beta) z\geq 0$ for $\beta \in [0,1]$. Hence,
$
r_\sigma'(\beta) = -\EE[(\sigma(z)-\sigma(\beta z))\sigma'(\beta z)z]\leq 0,
$
for any $\beta \in [0,1]$, i.e., $r_\sigma(\cdot)$ is monotonically decreasing in $[0,1]$. If $\sigma'(z)\geq C_1$ for $z\in [z_0,z_1]$, 
\begin{align*}
r'_\sigma(\beta)&\geq \frac{1}{\sqrt{2\pi}} \int_{z_0}^{z_1} (\sigma(z)-\sigma(\beta z))\sigma'(\beta z) z e^{-z^2/2}\mathrm{d}z\\
&\geq  \frac{1}{\sqrt{2\pi}}\int_{z_0}^{z_1} C_1(z-\beta z) z e^{-z^2/2}\mathrm{d}z=C_2(1-\beta),
\end{align*}
where $C_2=\frac{C_1}{\sqrt{2\pi}}\int_{z_0}^{z_1}z^2e^{-z^2/2}\mathrm{d}z$.
\end{proof}

The condition that $\sigma'(\cdot)$ is bounded away from zero in a neighbor of the origin is satisfied by all the monotonic activations used in practice. We remark that this condition is also necessary, otherwise $r_\sigma(\cdot)$ could be flat in some place of $[0,1]$. Consider the activation function $\sigma(z)=\max(1,\max(z-1,0))$, for which $\sigma'(z)=0$ for  $z\in (-\infty,1)$. Figure \ref{fig: monotonic-counter-example} shows the landscapes of $\sigma(\cdot)$ and $r_\sigma(\cdot)$. One can see that $r'_\sigma(\beta)=0$ when $\beta$ is close to $0$, which causes that GD starting from $\beta=0$ gets trapped, thereby failing to converge.

\begin{figure}[!h]
\centering 
\includegraphics[width=0.19\textwidth]{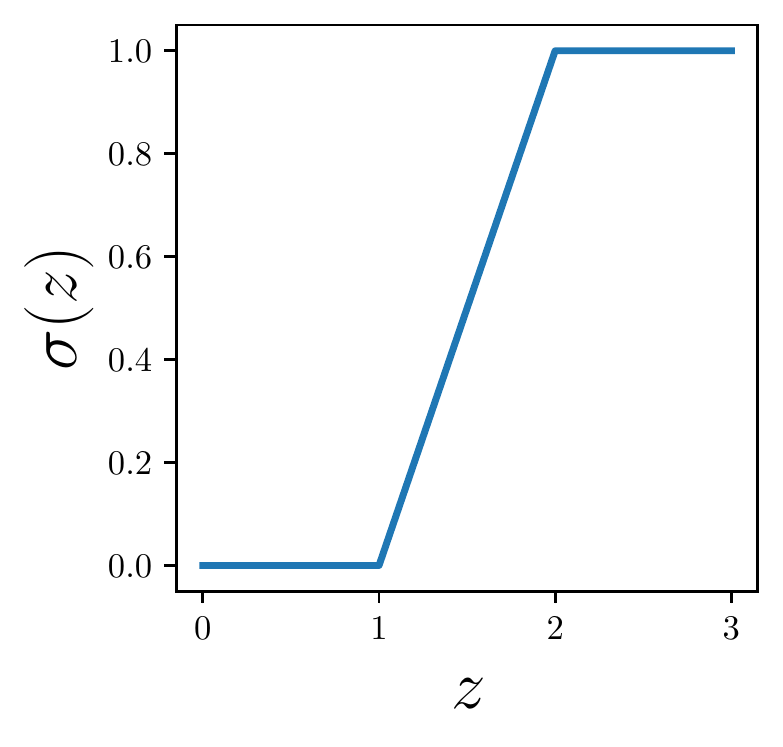}
\includegraphics[width=0.2\textwidth]{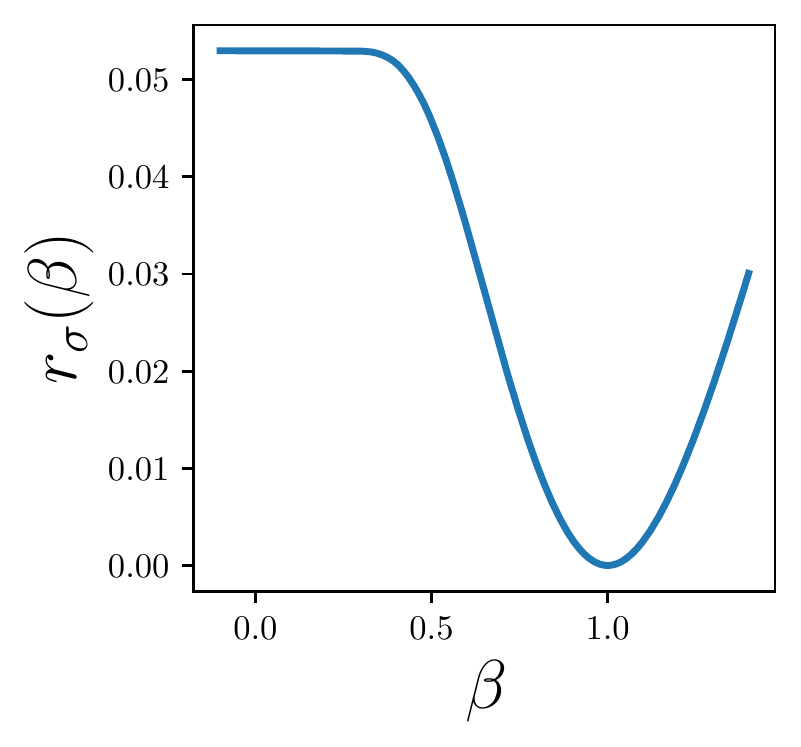}
\vspace*{-2mm}
\caption{$\sigma(z)=\max(1,\max(z-1,0))$ (left) and $r_\sigma(\cdot)$ (right). }
\label{fig: monotonic-counter-example}
\end{figure}

\subsubsection{Non-monotonic activations}
We now consider non-monotonic activation functions. 
\begin{assumption}\label{assumption-2}
 There exists  $z_0>0$ such that $\sigma(\cdot)$ is monotonically decreasing in $[-\infty,-z_0]$ and monotonically increasing in $[z_0,\infty]$. Moreover, we assume that there exist a $C>0$ such that $\sigma'(z)\geq C$ for $z\in [0,z_0]$, and  $q(z)=\sigma(z)-\sigma(-z), p(z)=\sigma(z)+\sigma(-z)$ are both monotonically increasing for $z\geq 0$.
\end{assumption}
The monotonicity of $q(\cdot)$ and $p(\cdot)$ ensure that the increasing part  dominates the decreasing part. The above assumption is  satisfied by  the self-gated family $\sigma(z)=z\phi(z)$ with  $\phi(z)+\phi(-z)=1$. In particular, SiLU/Swish and GELU belongs to this family.
This can be seen as follows. For any $z\geq 0$, 
$
q'(z) = \phi(z)-\phi(-z) + z(\phi'(z)+\phi'(-z))\geq 0,
$
and $p(z)=z(\phi(z)+\phi(-z))=z$. 

\begin{lemma}\label{lemma: non-monotonic-1d}
Under Assumption \ref{assumption-2},  there exists a constant $C>0$ such that  $r'_\sigma(\beta)\leq - C(1-\beta)$ for any $\beta \in [0,1]$.
\end{lemma}
The proof is deferred to Appendix \ref{sec: app-sym}, which is similar to the proof of Lemma \ref{lemma: monotonic-pop-1d} but more dedicated. 


\paragraph*{Relationship with existing negative results}
As a complement to these positive results, here we provide an analysis of the negative example used in \cite{shamir2018distribution}, where  $\sigma(z)=\sin(d z)$.
Figure \ref{fig: bad-1d-landscape} shows the landscape of $r_{\sigma}(\cdot)$ for various $d$'s. When $d=1$, the landscape is nice. However, when $d=2$, a bad local minimum appears in $[0,1]$. The situation becomes severer as increasing $d$.
Hence, GD with zero initialization fails to converge when $d$ is relatively large.
\begin{figure}[!h]
\centering
\includegraphics[width=0.47\textwidth]{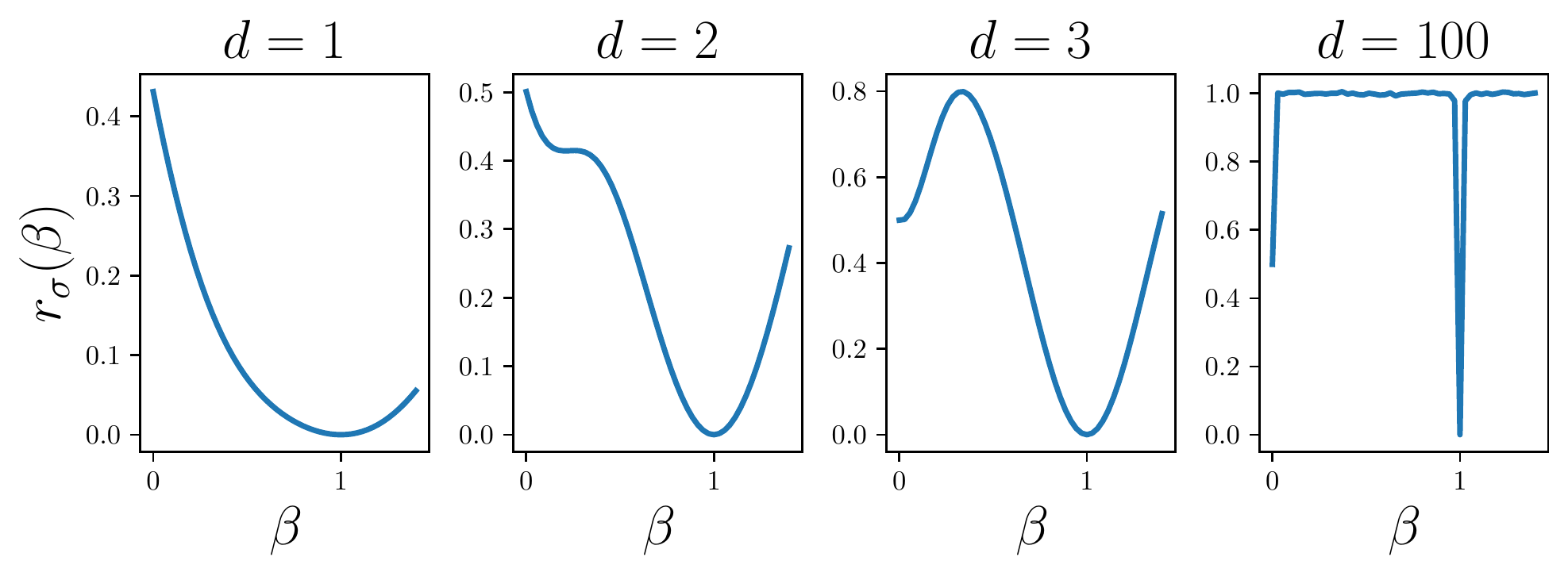}
\vspace*{-2mm}
\caption{The landscape of $r_{\sigma}(\cdot)$ for $\sigma(z)=\sin(d z)$.}
\label{fig: bad-1d-landscape}
\end{figure}

\subsection{Random Initialization}
\vspace*{-2mm}
In this section, we assume  $\bw\in\SS^{d-1},\sigma\in L^2(\mu_0)$ where $\mu_0=\mathcal{N}(0,1)$ and consider the random initialization $\bw_0\sim \text{Unif}(\SS^{d-1})$.
Let $\{h_i\}_{i=1}^{\infty}$ denote the probabilistic Hermite polynomials, which form a set of orthonormal basis of $L^2(\mu_0)$. In particular,
\[
    h_0(z)=1,\, h_1(z)=z,\,  h_2(z)=\frac{z^2-1}{\sqrt{2}}, h_3(z)=\frac{z^3-3z}{\sqrt{6}}.
\]
We expand  $\sigma$ as $
    \sigma(z) = \sum_{i=0}^{\infty} \hat{\sigma}_i h_i(z),
$
where $\hat{\sigma}_i=\EE_{z\sim\cN(0,1)}[\sigma(z)h_i(z)]$ is the \textit{Hermite coefficient} of $\sigma$. We will study how the decay of $\hat{\sigma}_i$ affects the property of the risk landscape and converge of GD. 
\begin{lemma}\label{lemma: risk-expression}
Assume that $\bw\in\SS^{d-1}$ and let $f(z)=\sum_{i=0}^{\infty} \hat{\sigma}_i^2 z^i$.  The population risk can be written as
\begin{equation}\label{eqn: sphere-pop-risk}
\cR(\bw) = f(1) - f(\bw^T\bw^*).
\end{equation}
\end{lemma}
\begin{proof}
Notice that 
$
\cR(\bw)=\frac{1}{2}\EE[\sigma(\bw^T\bx)^2]-\EE[\sigma(\bw^T\bx)\sigma({\bw^*}^T\bx)]+\frac{1}{2}\EE[\sigma({\bw^*}^T\bx)^2]
$
and for any $\bw_1,\bw_2\in\SS^{d}$, 
\begin{align}
\notag \EE_{\bx}&\left[\sigma(\bw_1^T\bx)\sigma(\bw_2^T\bx)\right]\\
\notag &=\EE\big[\sum_{i=0}^\infty \hat{\sigma}_i h_i(\bw_1^T\bx) \sum_{j=0}^\infty \hat{\sigma}_j h_j(\bw_2^T\bx)\big]\\
\nonumber   &= \sum_{i,j=0}^{\infty} \hat{\sigma}_{i} \hat{\sigma}_j\EE\left[h_i(\bw_1^T\bx)   h_j(\bw_2^T\bx)\right]\\
&= \sum_{i=0}^{\infty} \hat{\sigma}_i^2 (\bw_1^T\bw_2)^i,
\end{align}
where the last equality follows from \citep[Proposition 11.31]{o2014analysis}. 
\end{proof}

Denote by $\grad$  the Riemannian gradient on $\SS^{d-1}$. Then,
$
    \grad \cR(\bw) =  - (1-\bw\bw^T) f'(\bw^T\bw^*) \bw^*
$
and the GD flow on the sphere is given by 
\begin{equation}\label{eqn: riemanian-GD}
    \dot{\bw}_t = (1-\bw_t\bw_t^T) f'(\bw_t^T\bw^*) \bw^*.
\end{equation}
Let $a_t=\langle \bw_t,\bw^*\rangle$. Then, we have
\begin{align}\label{eqn: sphere-xxx}
  \dot{a}_t &= f'(a_t)(1-a_t^2),
\end{align}
which is an one-dimensional ODE,  completely determined by 
$
   f'(a)=\sum_{i=1}^\infty \hat{\sigma}_i^2 i a^{i-1}.
$
By \eqref{eqn: sphere-xxx}, the set of critical points of $\cR(\cdot)$ is given  by 
\begin{equation}
\mathcal{C}:=\{\bw\in\SS^{d-1}: f'(\bw^T\bw^*)=0 \text{ or } |\bw^T\bw^*|^2=1\}.
\end{equation}

\begin{remark}
Here we only consider  the Riemannian GD flow; otherwise, the $\mathbf{w}_t$ will leave away from $\mathbb{S}^{d-1}$, for which the risk landscape has a simple analytic expression. If we do not impose this constraint, the population landscape still has an analytic expression:
\[
    \mathcal{R}(\mathbf{w}) = \frac{1}{2}H(1,1,1) + \frac{1}{2} H(1,\|\mathbf{w}\|, \|\mathbf{w}\|)- H(\hat{\mathbf{w}}^T\mathbf{w}^*, \|\mathbf{w}\|,1),
\]
where $H:\RR^3\mapsto\RR$ is given by $H(z,s_1,s_2)=H(z,s_2,s_1)=\sum_{k=0}^\infty \hat{\sigma}_{k}(s_1)\hat{\sigma}_{k}(s_2)z^k$ and $\hat{\sigma}_k(s)=\mathbb{E}_{z\sim\mathcal{N}(0,1)}[\sigma(sz)h_k(z)]$.  In such a case, the analysis is much more involved since we need to characterize how the Hermite coefficients are affected by the dilation of $\sigma$. We leave this to future work.
\end{remark}

\subsubsection{Convergence with Constant Probability}

When $\sigma(\cdot)$ is nonzero, there must exist $i\in \NN_{+}$ such that $\hat{\sigma}_i^2>0$. Hence, $f'(a)\geq \hat{\sigma}_i^2 i a^{i-1}>0$ for $a> 0$. Consequently, the global minima $\bw=\bw^*$ is  unique critical point in the positive halfspace: $\{\bw\in\SS^{d-1}: \bw^T\bw^*>0\}$. Moreover, it is obvious that the whole positive halfspace is the basin of attraction. Using this observation, we have the following convergence result.

\begin{proposition}\label{pro: constant-prob-sphere}
Assume that $\sigma(\cdot)$ is nonzero. 
Let $k=\min\{i:\sigma_i\neq 0\}$. Then, there exists a constant $C>0$ such that for any $\delta \in (0,\frac 1 2)$, with probability $\frac{1}{2} - \frac{C \delta}{\sqrt{d}}$, we have $1-\bw_t^T\bw^*\leq e^{-c_k t}$ with $c_k=k\hat{\sigma}^2_k(\frac \delta d)^{k-1}$.
\end{proposition}
\begin{proof}
Since $\bw_0\sim \text{Unif}(\SS^{d-1})$, $a_0=\bw_0^T\bw^*$
follows the  distribution:
$
    g(z)=\frac{1}{\sqrt{\pi}}\frac{\Gamma(\frac d 2)}{\Gamma(\frac{d-1}{2})}(1-z^2)^{\frac{d-3}{2}}.
$
It is easy to verify that there exists a constant $C_1>0$ such that $(1-t)^q\leq 1-C_1 q t$ for $t\in [0,\frac 1 d]$. Then, for $\delta\leq 1$, 
\begin{align}\label{eqn: random-init-cp-1}
\notag \PP\{a_0\geq \frac \delta d \}&= \frac 1 2 - \int_0^{\frac \delta d} g(z)\mathrm{d}z\\
\notag &\geq \frac12 - \frac{\Gamma(\frac d 2)}{\sqrt{\pi}\Gamma(\frac{d-1}{2})}\int_0^{\frac \delta d} (1-c_1\frac{d-3}{2} z^2) \mathrm{d}z\\
&\geq \frac{1}{2}-C_2 \frac{\delta}{\sqrt{d}}.
\end{align}
Therefore, with probability $\frac{1}{2}-\frac{C_2\delta}{\sqrt{d}}$, $f'(a_0)\geq k\hat{\sigma}_k^2a_0^{k-1}>0$. With this initialization, $a_t$ keep increasing for $t\geq 0$. Then, we have $\dot{a}_t= f'(a_t)(1-a_t^2)\geq f'(a_0)(1-a_t)$. This yields that $1-a_t\leq e^{-f'(a_0)t}$. We thus complete the proof since $f'(a_0)\geq k\hat{\sigma}_k^2a_0^{k-1}$.
\end{proof}

Proposition \ref{pro: constant-prob-sphere} provides a constant-probability (close to $1/2$) guarantee for the GD convergence, and it only require $\sigma$ to be nonzero.  Moreover, the more the Hermite coefficients concentrate at small $k$'s, the faster is the convergence. In particular, if $\hat{\sigma}_1\neq 0$, we have $c_k=\hat{\sigma}_1^2$ and as such, the convergence rate is independent of $d$. 


\begin{figure}[!h]
\centering 
\includegraphics[width=0.27\textwidth]{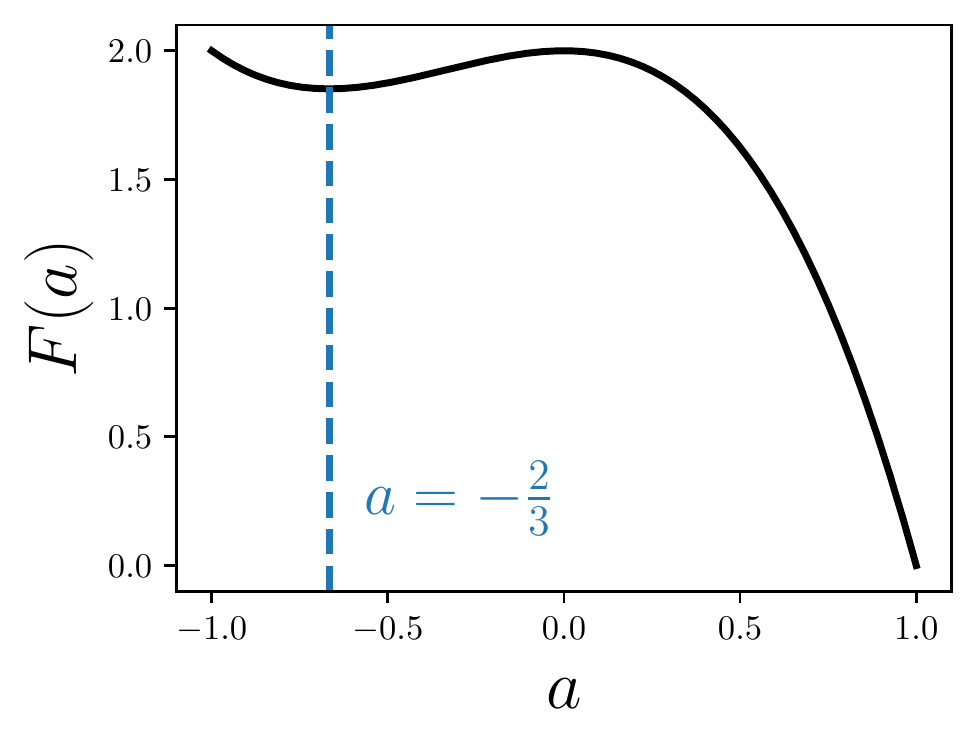}
\vspace*{-3mm}
\caption{The landscape $F(a)=1-f(a)$ for $\sigma=h_2+h_3$. Here $a=-2/3$ is a bad local minima.}
\label{fig: random-init-counter-exm}
\end{figure}

\paragraph*{Optimality.} The following lemma shows that the success probability  cannot be further improved  without imposing stronger conditions on $\sigma(\cdot)$. 
\begin{lemma}\label{lemma: random-init-optimality}
Assume $\sigma=h_2+h_3$, where $h_2$ and $h_3$ are the 2-th and $3$-th Hermite polynomial, respectively. Then, $\cR(\cdot)$ has bad local minima: $\cQ=\{\bw\in\SS^{d-1}: \bw^T\bw^*=-2/3\}$ and moreover, \wp $\frac{1}{2}-O(\frac{1}{\sqrt{d}})$ over the random initialization, GD converges to $\cQ$.
\end{lemma}

\begin{proof}
By the assumption,
$
    f(a) = a^2 + a^3, \, f'(a)=2a+3a^2.
$ Then, $\cR(\bw)=F(\bw^T\bw^*)$ with $F(a)=2-a^2-a^3$.
$F$ has a bad local minimum at $a=-2/3$, where $F(-2/3)=50/27>F(1)=0$ (See Figure \ref{fig: random-init-counter-exm} for an illustration).
Hence, $\{\bw\in\SS^{d-1}: \bw^T\bw^*=-2/3\}$ is a set of bad local minima of $\cR(\cdot)$. Substituting $f'(a)=2a+3a^2$ into \eqref{eqn: sphere-xxx} gives us 
\[
\dot{a}_t = a_t(2+3a_t)(1-a_t^2).
\]
Following the estimate \eqref{eqn: random-init-cp-1} and symmetry, we have \wp $1/2-C/(4\sqrt{d})$ that $-1/4\leq a_0< 0$. This will cause that $a_t$ decreases to $a=-2/3$. Therefore, when $d\gg 1$, with a probability close to $1/2$, GD fails to converge to global minima.
\end{proof}

\subsubsection{A High-Probability Convergence}
In this section, we show that the probability of GD convergence can be boosted (to $1$) by making stronger assumptions on the activation function. 

Let us first take a closer look at the risk landscape. Define 
\begin{equation}
q_\sigma(\delta)=\hat{\sigma}_1^2 - \sum_{i=1}^\infty (2i)\hat{\sigma}_{2i}^2\delta^{2i-1}.
\end{equation}
According to \eqref{eqn: sphere-xxx}, when $f'(a)>0$ for any $a\in [-1,0]$, there are only two critical points $\bw=\bw^*$ (minimum) and $\bw=-\bw^*$ (maximum).
One condition to ensure $
f'(a)>0,\,\forall a\in [-1,0]$ is $q_\sigma(1)>0$ since
\begin{equation}\label{eqn: hpc-1}
f'(a)\geq \hat{\sigma}_1^2 - \sum_{i=1}^\infty (2i)\hat{\sigma}_{2i}^2=q_\sigma(1)>0.
\end{equation}
Since $\hat{\sigma}_1=\EE_{z\sim\cN(0,1)}[z\sigma(z)]$, this condition implies that the linear component of $\sigma$ dominates the high-order components. We numerically verify that, $q_\sigma(1)>0$ for all the ReLU variants, including the non-monotonic SiLU/Swish and GELU. 

The above landscape analysis implies that when $q_\sigma(1)>0$, the success probability of convergence for random initialization is exactly $1$. The proposition given below further shows that as long as $q_\sigma(\delta)>0$ for some small constant $\delta>0$ is sufficient to establish a high-probability convergence when $d\gg 1$. Note that under this condition, there may exist bad local minima and saddle points. The high-probability convergence is made possible by two facts: (1) The near-origin region lies in the basin of attraction of the global minimum; (2) The random initialization can avoid the pathologic region with a high probability. 
\begin{proposition}\label{pro: high-prob}
Suppose that $q_\sigma(\delta)>0$ for some constant $\delta\in (0,1/2]$. Then,  with probability at least $1-0.5e^{-d\delta^2}$,  $1-\bw_t^T\bw^*\leq e^{-q_\sigma(\delta) t/2}$.
\end{proposition}

\begin{proof}
Notice that for $a\in [-\delta, 1]$,
\begin{equation}\label{eqn: xxxx}
    f'(a)=\hat{\sigma}_1^2 + 2\hat{\sigma}_2^2 a + 3\hat{\sigma}_3^2 a^2 + \dots\geq q_\sigma(\delta)>0.
\end{equation}
Since $\bw_0\sim \text{Unif}(\SS^{d-1})$,  with probability $1-0.5e^{-d\delta^2}$, $a_0=\bw^T_0\bw^*\geq -\delta$. Thus, $f'(a_0)\geq q_\sigma(\delta)>0$.  Therefore, $a_t$ is increasing for $t\in [0,\infty)$, and by using \eqref{eqn: xxxx},
$$
\dot{a}_t\geq q_\sigma(\delta) (1-a_t^2) = q_\sigma(\delta)(1-a_t)(1+a_t) \geq \frac{q_\sigma(\delta)}{2}(1-a_t)
$$
This leads to that $1-a_t\leq e^{-q_\sigma(\delta) t/2}$.
\end{proof}

\begin{remark}
Combined with Lemma \ref{lemma: random-init-optimality}, it is revealed that the dominance of linear component for the activation function is crucial for achieving  high-probability convergence. This provides an explanation of the wide use of ReLU and its variants.
\end{remark}

\paragraph*{Relationship with existing negative results}
Consider the setting used in \cite{shamir2018distribution}, where $\cD=\cN(0,I_d)$ and $\sigma(z)=\sin(dz)$. A detailed calculation (provided in Appendix \ref{sec: appendix-hermite-1}) tells us 
\begin{align}
\hat{\sigma}_1 = \EE_{z\sim \cN(0,1)}[z\sin(d z)] = d e^{-d^2/2}.
\end{align}
Hence, $\hat{\sigma}_1$ and $q_\sigma(\delta)$ are exponentially small.  Consequently, the convergence of GD is exponentially slow. Note that it is not surprising that $\hat{\sigma}_1$ and $q_\sigma(\delta)$ are exponentially small since  the activation function is highly oscillated in this case.

\section{Learning with Finite Samples}
\vspace*{-2mm}
\label{sec: empirical-GD}

We now proceed to the finite sample case. Specifically, we focus on the case that the input distribution is standard Gaussian. The extension to the setting used in Section \ref{sec: non-symmetric} is straightforward.
We make the following assumption for technical simplicity, which is satisfied by SiLU/Swish and GELU.
\begin{assumption}\label{assump-2: activation-function}
Assume that $\sigma',\sigma''$ exist and $\max(|\sigma'(z)|,|\sigma''(z)|)\lesssim 1$ for any $z\in\RR$.
\end{assumption}

Let $E_Q=\{\bw\in\RR^d: \|\bw-\bw^*\|\leq Q\}$. The following proposition bounds the difference between the empirical  and population landscape for $\bw\in E_Q$. 
\begin{proposition}\label{lemma: emp-pop-gap}
Assume that $n\geq 10$.
For any $\delta \in (0,1)$, $\wp$ $1-\delta$ over the sampling of training set, 
\begin{align*}
\sup_{\bw\in E_Q} |\hat{\cR}_n(\bw)-\cR(\bw)|&\lesssim \frac{\log(n/\delta))d}{\sqrt{n}} (Q+1)^2\\
\sup_{\bw\in E_Q} \|\nabla \hat{\cR}_n(\bw)-\nabla \cR(\bw)\|&\lesssim \frac{\log^{3/2}(n/\delta)\sqrt{d}}{\sqrt{n}}(Q+1)^2.
\end{align*}
\end{proposition}
This proposition is proved by using the techniques of empirical processes. However, the empirical processes in our case are not sub-gaussian  due to  the squared loss and the unboundedness of the input distribution. To handle this issue, we adopt a truncation method to capture the tail behavior. We refer to Appendix \ref{sec: proof-empirical-pop} for more details.

 The following lemma shows that the population risk and its gradient are Lipschitz continuous and the Lipschitz constants are independent of $d$. The proof is deferred to Appendix \ref{sec: app-lemma-lipschitz}.
\begin{lemma}\label{lemma: Lipschitz-gradient}
For any $\bw_1,\bw_2\in E_Q$, we have $|\cR(\bw_1)-\cR(\bw_2)|\lesssim Q\|\bw_1-\bw_2\|$ and $\|\nabla \cR(\bw_1)-\nabla \cR(\bw_2)\|\lesssim (1+Q)\|\bw_1-\bw_2\|$.
\end{lemma}

Using  Proposition \ref{lemma: emp-pop-gap} and Lemma \ref{lemma: Lipschitz-gradient}, we can convert the preceding convergence results of population GD to the empirical GD as shown below. The proofs are deferred to Appendix \ref{sec: app-proof-empirical-GD}.
 \begin{proposition}[Zero initialization]\label{pro: empirical-GD-zero-init}
Suppose that the activation function satisfies Assumption \ref{assump-2: activation-function} and the condition in Theorem \ref{thm: pop-GD}. Let $\hat{\bw}_t$ be the GD solution starting from zero. There exists $C_1,C_2,C_3>0$ and let   $\epsilon_n = \frac{C_3 \sqrt{d}\log^{3/2}(n/\delta)}{\sqrt{n}}$. There exists $T=\frac{\log(1/\epsilon_n)}{C_1+C_2}$ such that 
 \begin{equation}
    \|\hat{\bw}_{T}-\bw^*\|\leq \epsilon_n^{\frac{C_1}{C_1+C_2}}
\end{equation}
 \end{proposition}

\begin{proposition}[Random initialization]\label{pro: empirical-GD-sphere}
Let $\delta_1\in (0,1/2], \delta_2\in (0,1)$. Suppose that Assumption \ref{assump-2: activation-function} holds and $q_\sigma(\delta_1)>0$. Let $\hat{\bw}_t$ be the solution of the Riemaninan GD \eqref{eqn: riemanian-GD}  initialized from $\hat{\bw}_0\sim \text{Unif}(\SS^{d-1})$.
Then, \wp at least $1-0.5e^{-d\delta_1^2}$ over the initialization and  $1-\delta_2$ over the sampling of training set, we have 
\[
    \|\hat{\bw}_t - \bw^*\|^2\lesssim e^{- \frac{q_\sigma(\delta_1)}{2} t } + \frac{1}{q_\sigma(\delta_1)}\sqrt{\frac{d\log^3(n/\delta_2)}{n}}.
\]
\end{proposition}
The above two propositions show that learning a single neuron via GD only requires polynomial samples and polynomial time.  For instance, in  Proposition \ref{pro: empirical-GD-sphere}, the sample and time complexities are $\tilde{O}(d/\varepsilon^2)$ and $O(\log(1/\varepsilon))$, respectively. It should be stressed that our upper bounds  are not necessarily optimal and the logarithmic terms can be removed by assuming the input distribution to be bounded.

\section{Conclusion}
\vspace*{-2mm}
In this work, the problem of learning a single neuron with GD  is studied under the realizable setting. We show that a single neuron can be learned efficiently (i.e., the sample complexity and time complexity are polynomial in the input dimension and target accuracy) as long as the activation function has a dominating linear or monotonic component.  In contrast to existing work, our conditions remove the restriction of monotonicity and are satisfied by all the commonly-used non-monotonic activation functions.  It is of much interest to extend our analysis to the agnostic learning setting \citep{frei2020agnostic}, where no relationship between the label $y$ and the input $\bx$ is assumed. In such a case, one needs to deal with some extra hardness \citep{goel2019time}. For example, there may exist many bad local minima even if $\sigma$ is strictly monotonic \citep{auer1996exponentially}.


\subsubsection*{Acknowledgements}
We thank Weinan E, Chao Ma, and Jihao Long for many helpful discussions and anonymous reviewers for valuable suggestions.

\bibliographystyle{abbrvnat}
\bibliography{ref}

\clearpage
\appendix

\thispagestyle{empty}

\onecolumn \makesupplementtitle

\section{Proofs for Section \ref{sec: non-symmetric}}
\label{sec: app-non-sym}

\subsection{Proof of Proposition \ref{pro: non-symmetric-grad}}
Our proof needs the following technical lemma.
\begin{lemma}[Lemma B.1 in \citep{yehudai2020learning}]\label{lemma: non-symmetric}
For some fixed $\alpha$, and let $\ba,\bb$ be two unit vectors in $\RR^2$ such that $\arccos(\ba^T\bb)\leq \pi-\delta$ for some $\delta\in (0,\pi]$. Then,
\[
    \inf_{\bu\in\RR^2, \|\bu\|=1}\int \mathbf{1}_{\ba^T\by>0} \mathbf{1}_{\bb^T\by>0} \mathbf{1}_{\|\by\|\leq \alpha} (\bu^T\by)^2 \mathrm{d}\by \geq \frac{\alpha^4}{8\sqrt{2}}\sin^3\left(\frac{\delta}{4}\right)
\]
\end{lemma}
\paragraph*{Proof of Proposition \ref{pro: non-symmetric-grad}.}
Let $S(\bw,\bw^*)=\{\bx\in\RR^d: \bw^T\bx\geq 0, {\bw^*}^T\bx\geq 0, \|\bx\|\leq \alpha\}$ 
where  $\alpha$ is the constant defined in Assumption \ref{assumption: non-symmetric}, and 
\begin{align*}
A(\bw,\bw^*,\bx)=(\sigma(\bw^T\bx)-\sigma({\bw^*}^T\bx))\sigma'(\bw^T\bx)(\bw^T\bx-{\bw^*}^T\bx).
\end{align*}
Denote by $S^{c}(\bw,\bw^*)$ be the complement of $S(\bw,\bw^*)$. 
Using Assumption \ref{assumption: non-symmetric} and the mean value theorem, we have 
\[
A(\bw,\bw^*,\bx) \geq  \begin{cases}
 \gamma^2 (\bw^T\bx-{\bw^*}^T\bx)^2, & \text{if } \bx \in S(\bw,\bw^*)\\
- \zeta^2 (\bw^T\bx-{\bw^*}^T\bx)^2, &\text{if }\bx\in S^c(\bw,\bw^*).
\end{cases}
\]
Then, 
\begin{align}\label{eqn: app-0}
\notag \langle \nabla \cR(\bw), \bw - \bw^*\rangle &= \EE_{\bx}[A(\bw,\bw^*,\bx)]=\EE_{\bx}[A(\bw,\bw^*,\bx)\mathbf{1}_{S(\bw,\bw^*)}] + \EE_{\bx}[A(\bw,\bw^*,\bx)\mathbf{1}_{S^c(\bw,\bw^*)}] \\
\notag &\geq \gamma^2 \EE_{\bx}[(\bw^T\bx-{\bw^*}^T\bx)^2\mathbf{1}_{S(\bw,\bw^*)}] -\zeta^2 \EE_{\bx}[(\bw^T\bx-{\bw^*}^T\bx)^2\mathbf{1}_{S^c(\bw,\bw^*)}]\\
\notag &\geq (\gamma^2+\zeta^2)\EE_{\bx}[(\bw^T\bx-{\bw^*}^T\bx)^2\mathbf{1}_{S(\bw,\bw^*)}] - \zeta^2\EE_{\bx}[(\bw^T\bx-{\bw^*}^T\bx)^2]\\
&\geq  (\gamma^2+\zeta^2)\|\bw-\bw^*\|^2 \inf_{\bu\in \text{span}(\bw,\bw^*),\|\bu\|=1}\EE_{\bx}[(\bu^T\bx)^2\mathbf{1}_{S(\bw,\bw^*)}] - \zeta^2 \tau \|\bw-\bw^*\|^2,
\end{align}
where the last inequality uses the assumption that $\EE[\bx\bx^T]\leq \tau I_d$. What remains is to bound the first term of the right hand side. Let $\by=(\bw^T\bx, {\bw^*}^T\bx)\in\RR^2$ be the projection of $\bx$ into $\text{span}\{\bw,\bw^*\}$. Then,
\begin{align}\label{eqn: app-1}
\notag \inf_{\bu\in \text{span}(\bw,\bw^*),\|\bu\|=1}&\EE_{\bx}[(\bu^T\bx)^2 \mathbf{1}_{S(\bw,\bw^*)}] = \inf_{\bu\in \text{span}(\bw,\bw^*),\|\bu\|=1} \EE_{\bx}\left[(\bu^T\bx)^2 \mathbf{1}_{\|\bx\|\leq \alpha} \mathbf{1}_{\bw^T\bx\geq 0} \mathbf{1}_{{\bw^*}^T\bx\geq 0} \right]\\
\notag &\geq  \inf_{\bu\in\RR^2, \|\bu\|=1} \int (\bu^T\by)^2 \mathbf{1}_{\|\by\|\leq \alpha} \mathbf{1}_{y_1\geq 0} \mathbf{1}_{y_2\geq 0} p_{\bw,\bw^*}(\by) \mathrm{d}\by\\
\notag &\geq \beta \inf_{\bu\in\RR^2, \|\bu\|=1} \int (\bu^T\by)^2 \mathbf{1}_{\|\by\|\leq \alpha} \mathbf{1}_{y_1\geq 0} \mathbf{1}_{y_2\geq 0} \mathrm{d}\by\\
&\geq \beta \frac{\alpha^4}{8\sqrt{2}}\sin^3(\delta/4),
\end{align}
where the last inequality follows from Lemma \ref{lemma: non-symmetric}. Plugging \eqref{eqn: app-1} into \eqref{eqn: app-0} completes the proof.




\subsection{Proof of Proposition \ref{pro: GD-convergence-nonsymmetric}}
\label{sec: app-sym}
First, if the initialization satisfies $\|\bw_0-\bw^*\|<1$, then we must have $\|\bw_t-\bw^*\|<1$ for any $t\geq 0$. Otherwise, we must have $t_0=\inf\{t: \|\bw_t-\bw^*\|\geq 1\}<\infty$. Then, $\|\bw_t-\bw^*\|<1$ for $t\in [0,t_0)$. According to Lemma \ref{lemma: general-small-than-1},  $\lambda(\delta_t)>0$ for $t\in [0,t_0)$. Hence, $\mathrm{d}\|\bw_t-\bw^*\|^2/\mathrm{d} t\geq -\lambda \|\bw_t-\bw^*\|^2\leq 0$ for $t\in [0,t_0)$, which implies that for any $t<t_0$, $\|\bw_{t}-\bw^*\|\leq \|\bw_0-\bw^*\|<1=\|\bw_{t_0}-\bw^*\|$. This is contradictory to the continuity of the GD trajectory.  Thus, $\|\bw_t-\bw^*\|^2\leq e^{-\lambda(\frac{\pi}{2})t}\|\bw_0-\bw^*\|^2$.

Second, according to \cite[Lemma 5.1]{yehudai2020learning}, with probability larger than $\frac{1}{2}-\frac{1}{4}\eta d-1.2^{-d}$, we have $\|\bw_0-\bw^*\|^2 \leq 1-2 \eta^2 d< 1$. Therefore, we complete the proof.

\section{Proofs of Section \ref{sec: symmetric}}
\subsection{Proof of Lemma \ref{lemma: non-monotonic-1d}}
Firstly, we can write $-r'_\sigma(\beta)=\frac{1}{\sqrt{2\pi}}\int_0^\infty a(\beta,z)ze^{-z^2/2}\mathrm{d}z$, where
\[a(\beta,z)= (\sigma(z)-\sigma(\beta z))\sigma'(\beta z) - (\sigma(-z)-\sigma(-\beta z))\sigma'(-\beta z).
\] 
\begin{itemize}
\item When $-\beta z\leq -z_0$, $(\sigma(-z)-\sigma(-\beta z))\sigma'(-\beta z)\leq 0$. Hence, $a(\beta,z)\geq (\sigma(z)-\sigma(\beta z))\sigma'(\beta z)\geq 0$.
\item When $-\beta z\geq -z_0$, we have $\sigma'(\beta z)\geq \sigma'(-\beta z)\geq 0$. Hence, using the the monotonicity of $q(\cdot)$, we have 
\[
    a(\beta, z)\geq \left[(\sigma(z)-\sigma(\beta z)) - (\sigma(-z)-\sigma(-\beta z))\right]\sigma'(-\beta z)=[q(z)-q(\beta z)]\sigma'(-\beta z)\geq 0,
\]
\end{itemize}
Combining them together, $a(\beta,z)\geq 0$ for any $z\geq 0, \beta\in [0,1]$. Hence, 
\begin{align*}
- r'(\beta) &=  \frac{1}{\sqrt{2\pi}}\int_0^{\infty} a(\beta,z) ze^{-\frac{z^2}{2}}\mathrm{d}z\geq \frac{1}{\sqrt{2\pi}}\int_{0}^{z_0} a(\beta,z) ze^{-\frac{z^2}{2}}\mathrm{d}z\\
&\geq \frac{1}{\sqrt{2\pi}}\int_{0}^{z_0} (\sigma(z)-\sigma(\beta z))\sigma'(\beta z)ze^{-\frac{z^2}{2}}\mathrm{d}z \geq C \int_0^{z_0} (z-\beta z) z e^{-\frac{z^2}{2}}\mathrm{d}z \geq C(1-\beta).
\end{align*}
\qed 

\subsection{Calculation of $\hat{\sigma}_1$ for the Sine activation function}
\label{sec: appendix-hermite-1}

\begin{align*}
    \hat{\sigma}_1 &= \EE_{z\sim\cN(0,1)}[z\sigma(z)] = \frac{1}{\sqrt{2\pi}} \int_{\RR} z \sin(dz) e^{-z^2/2} \mathrm{d}z = \frac{d}{\sqrt{2\pi}}\int\cos(dz)e^{-t^2/2}\mathrm{d} z\\
    &=\frac{d}{\sqrt{2\pi}}\sum_{n=0}^{\infty}\frac{(-1)^n}{(2n)!}\int (dt)^{2n} e^{-z^2/2}\mathrm{d}z\\
    &= d  \sum_{n=0}^{\infty} \frac{(-1)^nd^{2n}}{(2n)!} (2n-1)!!\\
    &=d \sum_{n=0}^{\infty} \frac{(-d^2/2)^n}{n!} = d e^{-d^2/2}.
\end{align*}
Therefore, the first Hermite coefficient is exponentially small for the periodic activation function: $\sigma(z)=\sin(dz)$.

\section{Proofs for empirical GD}
\label{sec: app-proof-empirical-GD}

\subsection{Proof of Proposition \ref{pro: empirical-GD-zero-init}}
Denote by $\bw_t$ and $\hat{\bw}_t$ the solutions of population and empirical GD, respectively, i.e., 
$\bw_0=\hat{\bw}_0=0$ and 
$
\dot{\bw}_t = - \nabla \cR(\bw_t),
\dot{\hat{\bw}}_t = -\nabla \hat{\cR}_n(\hat{\bw}_t).
$
By Theorem \ref{thm: pop-GD}, we have
\begin{equation}\label{eqn: pop-decay-1}
\|\bw_t-\bw^*\|\leq e^{-C_1 t}.
\end{equation}
For the empirical GD, let $T_0=\inf \{ t: \|\hat{\bw}_t-\bw^*\|\geq 2\}$ and $\Delta_t=\bw_t-\hat{\bw}_t$. Then, for $t\leq T_0$, 
\begin{align*}
\frac{\mathrm{d}\|\Delta_t\|^2}{\mathrm{d}t} &= - 2\langle \nabla \cR(\bw_t)  - \nabla \cR(\hat{\bw}_t),\Delta_t\rangle - 2\langle \nabla \cR(\hat{\bw}_t) - \nabla \hat{\cR}_n(\hat{\bw}_t), \Delta_t\rangle\\
&\lesssim \|\Delta_t\|^2 + \frac{\sqrt{d}\log^{3/2}(n/\delta)}{\sqrt{n}}\|\Delta_t\|,
\end{align*}
where the last inequality follows from Lemma \ref{lemma: Lipschitz-gradient} and Proposition \ref{lemma: emp-pop-gap}. Let $\epsilon_n = \frac{C_3 \sqrt{d}\log^{3/2}(n/\delta)}{\sqrt{n}}$.
Hence $
\frac{\mathrm{d}\|\Delta_t\|}{\mathrm{d}t}\leq C_2 \|\Delta_t\| + \epsilon_n,
$
which yields to 
\begin{equation}\label{eqn: empi-dev-2}
\|\Delta_t\|\leq \|\Delta_0\|+\epsilon_n (e^{C_2 t}-1)= \epsilon_n(e^{C_2 t}-1),
\end{equation}
where we use the fact that $\Delta_0=0$. Combining \eqref{eqn: pop-decay-1} and \eqref{eqn: empi-dev-2} leads to 
\begin{align}
\|\hat{\bw}_t -\bw^*\|&\leq \|\hat{\bw}_t-\bw_t\| + \|\bw_t-\bw^*\| \leq \epsilon_n(e^{C_2 t}-1) +  e^{-C_1 t}=: e(t)-\epsilon_n,
\end{align}
Taking $\epsilon_n e^{C_1t}=e^{-C_2 t}$ gives $T=\frac{\log(1/\epsilon_n)}{C_1+C_2}$. Obviously, $e(\cdot)$ is monotonically decreasing for $t\leq T$. Thus, for $t\leq T$, 
$\|\hat{\bw}_t-\bw^*\|\leq e(t)-\epsilon_n\leq  e(0)-\epsilon_n = 1$.
Therefore, we must have $T_1\leq T_0$. This means that the previous estimates hold for $t\leq T$. Taking $t=T$, we have 
$
    \|\hat{\bw}_{T}-\bw^*\|\leq e(T)-\epsilon_n \lesssim \epsilon_n^{\frac{C_1}{C_1+C_2}}
$

\subsection{Proof of Proposition \ref{pro: empirical-GD-sphere}}
The empirical GD can be written as 
\[
    \hat{\bw}_t = -(I-\hat{\bw}_t\hat{\bw}_t^T)\nabla \cR(\hat{\bw}_t) - (I-\hat{\bw}_t\hat{\bw}_t^T)(\nabla \hat{\cR}_n(\hat{\bw}_t)-\nabla \cR(\hat{\bw}_t)).
\]
 Let $\hat{a}_t=\langle \hat{\bw}_t, \bw^*\rangle$ and $e_t=-{\bw^*}^T(I-\hat{\bw}_t\hat{\bw}_t^T)(\nabla \hat{\cR}_n(\hat{\bw}_t)-\cR(\hat{\bw}_t))$. Then,
\[
    \dot{\hat{a}}_t = f'(\hat{a}_t)(1-\hat{a}_t^2) + e_t,
\]
By Proposition \ref{lemma: emp-pop-gap} and $\|\hat{\bw}_t\|=1$, with probability $1-\delta_2$, we have $e_t\leq O(\sqrt{\frac{d\log^3(n/\delta_2)}{n}})=:\epsilon_n$. Analogous to the proof of Proposition \ref{pro: high-prob}, we have with probability $1-0.5e^{-d\delta^2_1}$ that, 
\[
\frac{d}{dt}(1-\hat{a}_t) \leq \frac{q_\sigma(\delta_1)}{2} (1-\hat{a}_t) + \delta_t\leq \frac{q_\sigma(\delta_1)}{2} (1-\hat{a}_t) + \epsilon_n.
\]
By Gronwall's inequality, 
$
1-\hat{a}_t\leq (1-\hat{a}_0)e^{-q_\sigma(\delta_1)t/2} + \frac{2\epsilon_n}{q_\sigma(\delta_1)}.
$

\section{Proof of Lemma \ref{lemma: Lipschitz-gradient}}
\label{sec: app-lemma-lipschitz}

For any $\bw\in E_Q$, consider the orthogonal decomposition: $\bw=\beta \bw^*+\alpha \bw_{\perp}$ with $\langle\bw_{\perp},\bw^*\rangle=0$ and $\|\bw_{\perp}\|=1$.
Let $V=(\bv_1,\bv_2,\dots,\bv_d)^T\in\RR^{d\times d}$ be an orthonormal matrix with $\bv_1=\bw^*,\bv_2=\bw_{\perp}$. Using change of variable $\bx=V\bx$ and the symmetry of $\cN(0,I_d)$, we have 
\[
    \nabla \cR(\bw) = \EE_{\bx}[(\sigma(\bw^T\bx)-\sigma({\bw^*}^T\bx))\sigma'(\bw^T\bx)\bx] = V^T \bu,
\] 
where 
$
\bu= \EE_{\bx}[(\sigma(\beta x_1+\alpha x_2)-\sigma(x_1))\sigma'(\beta x_1+\alpha x_2)\bx]=(u_1,u_2,0,\dots,0).
$
Here $u_i=\EE[(\sigma(\beta x_1+\alpha x_2)-\sigma(x_1))\sigma'(\beta x_1+\alpha x_2)x_i]$. Hence, it is easy to see that $\|\nabla \cR(\bw) \|=\|\bu\|\leq CQ$.

In addition,
\begin{align}
\notag \nabla^2 \cR(\bw) &= \EE[\sigma'(\bw^T\bx)\sigma'(\bw^T\bx)\bx\bx^T] + \EE[(\sigma(\bw^T\bx)-\sigma(\bw_*^T\bx))\sigma''(\bw^T\bx)\bx\bx^T]\\
&:= H_1 + H_2. 
\end{align}
We then estimate $H_1,H_2$ separately. By the symmetry of the input distribution,  $H_1=\EE[\sigma'(\|\bw\|x_1)^2\bx\bx^T]$. Hence 
\[
(H_1)_{i,j}=\EE[\sigma'(\|\bw\|x_1)^2x_ix_j] = 
\begin{cases}
0 & \text{if } i\neq j \\
\EE[\sigma'(\|\bw\|x_1)^2x_i^2] & \text{if } i=j.
\end{cases}
\]
Therefore, $H_1$ is diagonal and $\lambda_{\max}(H_1)\leq C $.
Let us turn to $H_2$. Consider the orthogonal decomposition: $\bw=\beta \bw^{*} +\alpha \bw_{\perp}$ with $\langle \bw_{\perp}, \bw^*\rangle=0$ and $\|\bw_{\perp}\|=1$. By symmetry, $H_2=\EE[(\sigma(\alpha x_1+\beta x_2)-\sigma(x_2))\sigma''(\alpha x_1+\beta x_2)\bx\bx^T]$. Let 
\begin{align}
\notag c_{s,t} &=\EE_{x_1,x_2\sim \cN(0,1)}[(\sigma(\alpha x_1+\beta x_2)-\sigma(x_2))\sigma''(\alpha x_1+\beta x_2)x_sx_t],\quad s,t=1,2\\
q &=\EE_{x_1,x_2,x_3\sim \cN(0,1)}[(\sigma(\alpha x_1+\beta x_2)-\sigma(x_2))\sigma''(\alpha x_1+\beta x_2)x^2_3]
\end{align}
Hence, 
\begin{align}
H_2
&=\begin{pmatrix}
c_{1,1} & c_{1,2} & 0 & \hdots & 0 \\
c_{2,1} & c_{2,2} & 0 & \hdots & 0 \\
0 & 0 & q& \hdots & 0\\
\vdots & \vdots & \vdots & \ddots & \vdots\\
0 & 0 & 0 & \hdots &q
\end{pmatrix}
\end{align}
It is easy to obtain that 
\begin{align}
\lambda_{\max}(H_2)\leq \max\{q, c_{1,1}+c_{2,2}\}\lesssim |\alpha|+|\beta-1|\lesssim \|\bw-\bw^*\|.
\end{align}
Combining the estimates of $H_1$ and $H_2$, we complete the proof.
\qed

\section{Proof of Proposition \ref{lemma: emp-pop-gap}}
\label{sec: proof-empirical-pop}
\subsection{Tool box for bounding empirical processes}

\begin{definition}\label{definition: orlic-norm}
Let $\psi$ be a nondecreasing, convex function with $\psi(0)=0$. 
The Orlicz norm of a random variable $X$ is defined by 
\[
\|X\|_{\psi}:=\inf\{t>0: \EE[\psi(|X|/t)]\leq 1\}.
\]
\end{definition}
For our purposes, Orlicz norms of interest are the ones given by $\psi_p(x)=e^{x^p}-1$ for $p\geq 1$. In particular, the cases of $p=1$ and $p=2$ correspond to the sub-exponential and sub-gaussian random variables, respectively. A random variable with finite $\psi_p$-norm has the following  control of the tail behavior
\[
    \PP\{|X|\geq t\}\leq C_1 e^{-C_2 \frac{t^p}{\|X\|^p_{\psi_p}}},
\]
where $C_1,C_2$ are constant that may depend on the value of $p$.

\begin{lemma}\label{lemma: pro-orlic-norm}
\begin{itemize}
    \item If $X\sim \cN(0,\sigma^2)$, $X$ is sub-gaussian with $\|X\|_{\psi_2}\leq C\sigma$.
\item Let $X,Y$ be sub-gaussian random variables. Then, $XY$ is sub-exponential and
\[
    \|XY\|_{\psi_1}\leq \|X\|_{\psi_2}\|Y\|_{\psi_2}.
\]
\item If $|X|\leq |Y|$ a.s., then $\|X\|_{\psi}\leq \|Y\|_{\psi}$ for any $\psi$ that satisfies the condition in Definition \ref{definition: orlic-norm}.
\end{itemize}
\end{lemma}

\begin{theorem}[Bernstein's inequality]\label{thm: bernstein}
Let $X_1,\dots,X_n$ be independent sub-exponential random variables. Suppose $K=\max_{i}\|X_i\|_{\psi_1}<\infty$. Then, for any $t>0$, we have 
\[
\PP\Big\{\big|\frac{1}{n}\sum_{i=1}^n X_i - \EE[X]\big|\geq t\Big\}\leq 2 \exp\left(- C n \min\left(\frac{t^2}{K^2}, \frac{t}{K}\right)\right).
\]
\end{theorem}

\begin{proposition}[Sums of independent sub-gaussians]\label{pro: sum-sub-gaussians}
Let $X_1,\dots,X_n$ be independent, mean zero, sub-gaussian random variables. Then, $\sum_{i=1}^n X_i$ is also a sub-gaussian random variable, and 
\[
    \|\sum_{i=1}^m X_i \|_{\psi_2}^2 \leq C \sum_{i=1}^n \|X_i\|_{\psi_2}^2.
\]
\end{proposition}

\begin{lemma}[Centering]
For a random variable $X$, we have $\|X-\EE[X]\|_{\psi_p}\leq C\|X\|_{\psi_p}$ for a constant $C>0$ that may depend on $p$.
\end{lemma}

We refer the reader to \citep[Section 2]{vershynin2018high} and \citep[Section 2]{van1996weak} for the proof of the above properties and more information on Orlicz spaces.

Let $(T,\rho)$ be a semi-metric space, i.e., $\rho(t_1,t_2)\leq \rho(t_1,t_3)+\rho(t_3,t_2)$ and $\rho(t_1,t_2)=\rho(t_2,t_1)$ for any $t_1,t_2,t_3\in T$. We denote the diameter of $T$ with respect to $\rho$ by  $\diam(T)=\sup_{s,t\in T}\rho(s,t)$. 
\begin{definition}[Sub-gaussian process]\label{definition: sub-gaussian-process}
Consider a random process $(X_t)_{t\in T}$ on a semi-metric space $(T,\rho)$. We say that the process is a sub-gaussian process if there exits $K\geq 0$ such that 
\[
    \|X_t-X_s\|_{\psi_2}\leq K \rho(t,s)\qquad \forall\,\, t,s\in T.
\]
\end{definition}

The following theorem gives a bound of a sub-gaussian process $(X_t)_{t\in T}$ in terms of the Dudley integral 
\[
J(\delta)=\int_\delta^{\diam(T)} \sqrt{\log N(T,\rho,\varepsilon)}\mathrm{d}\varepsilon,
\]
where $N(T,\rho,\varepsilon)$ is the $\varepsilon$-covering number of $T$ with respect to $\rho$.

\begin{theorem}[Theorem 8.1.6 in \citep{vershynin2018high}]\label{thm: uniform-bound}
Let $(X_t)_{t\in T}$ be a mean zero sub-gaussian process as in \ref{definition: sub-gaussian-process} on a semi-metric space $(T,\rho)$. Then, there exist $C>0$ such that for any $u>0$,  we have with probability $1-2e^{-u^2}$ that 
\begin{equation}
\sup_{t\in T} |X_t|\leq  C K \left(J(0)+\diam(T)u\right).
\end{equation}
\end{theorem}

\paragraph*{Some facts}
Here, we state some facts which will repeatedly used in the subsequent analysis.
Consider  the metric space $B_Q=\{\bw\in\RR^d: \|\bw-\bw^*\|\leq Q\}$ with $\|\cdot\|_2$.  Following Corollary 4.2.13 of \citep{vershynin2018high}, we have 
\begin{equation}\label{eqn: covering-number}
    N(B_Q,\|\cdot\|_2,\varepsilon)\leq \left(\frac{2Q}{\varepsilon}+1\right)^{d},
\end{equation}
where we omit the dependence on $\bw^*$ since it holds for any $\bw^*\in\RR^d$.

For any $M>0$ and $\bx\in\RR^d$, we define  $\bx^M:=\bx\min(1,\frac{M}{\|\bx\|})$. Hence, if $\|\bx\|\leq M$, $\bx^{M}=\bx$. 
Let $\bX\sim \cN(0,I_d)$. Then, $\|\bX\|^2=\sum_{i=1}^dX_i^2$ follows the $\chi_d^2$ distribution. 
Following Eq. (3.1) in \citep{vershynin2018high}, we have for $M\geq 2d$,
\begin{equation}\label{eqn: tail-chi-square}
    \PP\{\|\bX\|^2\geq M\}\leq 2 e^{-C M }.
\end{equation}
Moreover, for any $\bu\in\RR^d$, $|\bu^T\bX^M|=|\bu^T\bX\min(1,M/\|\bX\|)|\leq |\bu^T\bX|$. By Lemma \ref{lemma: pro-orlic-norm}, we have 
\begin{equation}\label{eqn: 11}
    \|\bu^T\bX^M\|_{\psi_2}\leq \|\bu^T\bX\|_{\psi_2}\leq C\|\bu\|.
\end{equation}

\subsection{Bounding the difference of loss function}
\label{sec: bounding-diff-risk}

In this subsection, we let $T_Q=B_{Q}(\bw^*)$ and $\rho(\bw_1,\bw_2)=\|\bw_1-\bw_2\|$.
Consider $f_{\bw}(\bx)=(\sigma(\bw^T\bx)-\sigma({\bw^*}^T\bx))^2$ and define the empirical process $(Z_{\bw})_{\bw\in T_Q}$:
\begin{align}
Z_{\bw}:=\hat{\cR}_n(\bw)-\cR(\bw) &= \frac{1}{n}\sum_{i=1}^n f_{\bw}(\bX_i) - \EE[f_{\bw}(\bX)],
\end{align}
where $\bX_i\stackrel{iid}{\sim}\cN(0,I_d)$. Define the truncated version as follows 
\vspace*{-2mm}
\begin{align}
Z^M_{\bw}= \frac{1}{n}\sum_{i=1}^n f_{\bw}(\bX_i^M) - \EE[f_{\bw}(\bX^M)],
\end{align}
Then, we can bound $(Z_{\bw})_{\bw\in T_Q}$ using the following decomposition
\begin{align}\label{eqn: truncate}
\sup_{\bw\in T_Q}|Z_{\bw}|\leq \sup_{\bw\in T_Q}|Z_{\bw}-Z_{\bw}^M| + \sup_{\bw \in T_Q}|Z_{\bw}^M|.
\end{align}
We will estimate the two terms of right hand side separately.

\begin{lemma}\label{lemma: risk-truncate}
For any $\bw\in T_Q$, we have 
\begin{align}\label{eqn: lip-pro-1}
\notag |f_{\bw}(\bx_1)-f_{\bw}(\bx_2)|&\leq (Q+1)^2(\|\bx_1\|+\|\bx_2\|)\|\bx_1-\bx_2\|\\
|f_{\bw_1}(\bx^M)-f_{\bw_2}(\bx^M)|&\leq 2QM|(\bw_1-\bw_2)^T\bx^M|.
\end{align}
\end{lemma}
\begin{proof}
We first have
\begin{align*}
|f_{\bw}(\bx_1)-f_{\bw}(\bx_2)|&=|(\sigma(\bw^T\bx_1)-\sigma({\bw^*}^T\bx_1))^2 - (\sigma(\bw^T\bx_2)-\sigma({\bw^*}^T\bx_2))^2|\\
&= (\sigma(\bw^T\bx_1)-\sigma({\bw^*}^T\bx_1)+\sigma(\bw^T\bx_2)-\sigma({\bw^*}^T\bx_2))\\
&\quad \cdot (\sigma(\bw^T\bx_1)-\sigma({\bw^*}^T\bx_1)-\sigma(\bw^T\bx_2)+\sigma({\bw^*}^T\bx_2))\\
&\leq (Q+1)^2(\|\bx_1\|+\|\bx_2\|)\|\bx_1-\bx_2\|,
\end{align*}
where the last inequality is due to that $\sigma$ is $1$-Lipschitz and $\|\bw-\bw^*\|\leq Q$.
Then,
\begin{align}
\notag |f_{\bw_1}(\bx^M)-f_{\bw_2}(\bx^M)| &=| (\sigma(\bw_1^T \bx^M)-\sigma({\bw_1^*}^T \bx^M))^2 - (\sigma(\bw_2^T \bx^M)-\sigma({\bw^*}^T \bx^M))^2| \\
\notag &= | (\sigma(\bw^T_1 \bx^M)+\sigma(\bw_2^T \bx^M)-2\sigma({\bw^*}^T \bx^M)) (\sigma(\bw^T_1\bx^M)-\sigma(\bw^T_2 \bx^M))|\\
\notag &\leq (|(\bw_1-\bw^*)^T \bx^M| + |(\bw_2-\bw^*)^T\bx^M|)|(\bw_1-\bw_2)^T\bx^M|\\
\notag &\leq 2Q M|(\bw_1-\bw_2)^T\bx^M|,
\end{align}
where the third inequality follows from that $\sigma$ is $1$-Lipschitz continuous.
\end{proof}
We then have the following bound of the first term on the right hand side of  \eqref{eqn: truncate}.
\begin{lemma}\label{lemma: trucated-error}
For any $\delta\in (0,1)$, with probability $1-\delta$ over the sampling of data, we have 
\begin{equation}
\sup_{\bw\in T_Q}|Z_{\bw}-Z_{\bw}^M|\leq C_1(Q+1)^2\left(d\max\left\{\sqrt{\frac{\log(2/\delta)}{n}}, \frac{\log(2/\delta)}{n}\right\} + e^{-C_2 M^2}\right)
\end{equation}
\end{lemma}
\begin{proof}
Using Lemma \ref{lemma: risk-truncate} and the fact, $\|\bX_i^M\|\leq \|\bX_i\|$, we have
\begin{align}
\notag |Z_{\bw} - Z_{\bw}^M| &\leq  \frac{1}{n}\sum_{i=1}^n |f_{\bw}(\bX_i) - f_{\bw}(\bX_i^M)| + \EE[|f_{\bw}(\bX)-f_{\bw}(\bX^M)|] \\
\notag &\leq \frac{2(Q+1)^2}{n}\sum_{i=1}^n \|\bX_i\|\|\bX_i-\bX_i^M\| + 2(Q+1)^2 \EE[\|\bX\|\|\bX-\bX^M\|]\\
&= \frac{2(Q+1)^2}{n}\sum_{i=1}^n (V_i^M - \EE[V^M]) + 4(Q+1)^2 \EE[V^M],
\end{align}
where we let  $V^M = \|\bX\|\|\bX-\bX^M\|=\|\bX\|^2(1-\min(1,M/\|\bX\|))$. Then,  
$$
\|V^M\|_{\psi_1}\leq \|\|\bX\|^2\|_{\psi_1}\leq C d.
$$
By Theorem \ref{thm: bernstein}, we have 
\begin{equation}\label{eqn: bound-dff}
\PP\Big\{\big|\frac{1}{n}\sum_{i=1}^n V_i^M - \EE[V^M]\big|\geq t\Big\}\leq 2 \exp\left(- C n \min\left(\frac{t^2}{d^2}, \frac{t}{d}\right)\right).
\end{equation}
By \eqref{eqn: tail-chi-square}, we have 
\begin{align}\label{eqn: expect-truncation}
\notag \EE[V^M] &= \int_0^\infty \PP\{V^M\geq t\}\mathrm{d}t =\int_0^{\infty}\PP\{\|\bX\|\left(\|\bX\|-\min\{M,\|\bX\|\}\right)\geq t\}\mathrm{d}t\\
\notag &=\int_0^{\infty}\PP\{\|\bX\|^2-M\|\bX\|\geq t\}\mathrm{d}t=\int_0^{\infty}\PP\{\|\bX\|\geq \sqrt{t+M^2/4}+M/2\}\mathrm{d}t\\
 &\leq \int_0^{\infty} 2e^{-C\big(\sqrt{t+M^2/4}+M/2\big)^2} \mathrm{d}t\leq 2e^{-CM^2/2}\int_0^{\infty} e^{-C t}\mathrm{d}t = \frac{2}{C}e^{-CM^2/2}.
\end{align}
Combining \eqref{eqn: bound-dff} and \eqref{eqn: expect-truncation} and taking RHS of \eqref{eqn: bound-dff}=$\delta$,  we complete the proof.
\end{proof}

We proceed to bound the second term on the right hand side of \eqref{eqn: truncate}.

\begin{lemma}\label{lemma: trucated-empirical-process}
For any $\delta\in (0,1)$, with probability $1-\delta$, we have 
\[
    \sup_{\bw \in T_Q}|Z_{\bw}^M|\lesssim \frac{MQ^2}{\sqrt{n}}(\sqrt{d}+\sqrt{\log(\delta/2)}).
\]
\end{lemma}
\begin{proof}

By Lemma \ref{lemma: pro-orlic-norm} and \ref{lemma: risk-truncate} , we have 
\begin{align}
\notag \|f_{\bw_1}(\bX^M)-f_{\bw_2}(\bX^M)\|_{\psi_2}\leq 2QM \|(\bw_1-\bw_2)^TX^M\|_{\psi_2} \leq C QM \|\bw_1-\bw_2\|.
\end{align}
where the last inequality is due to Eq. \eqref{eqn: 11}. By Proposition \ref{pro: sum-sub-gaussians}, we have 
\begin{align}
\|Z_{\bw_1}^M - Z_{\bw_2}^M\|_{\psi_2} &= \big\|\frac{1}{n}\sum_{i=1}^n (f_{\bw_1}(\bX_i^M) -f_{\bw_2}(\bX_i^M) - \EE[f_{\bw_1}(\bX^M)] - \EE[f_{\bw_2}(\bX^M)])\big\|_{\psi_2}\\
&\leq \frac{1}{n}\sqrt{\sum_{i=1}^n \|f_{\bw_1}(\bX_i^M) -f_{\bw_2}(\bX_i^M) - \EE[f_{\bw_1}(\bX^M)] - \EE[f_{\bw_2}(\bX^M)\|^2_{\psi_2}}\\
&\leq \frac{C}{n}\sqrt{\sum_{i=1}^n \|f_{\bw_1}(\bX_i^M) -f_{\bw_2}(\bX_i^M)\|^2_{\psi_2}}\\
&\leq \frac{CQM\|\bw_1-\bw_2\|}{\sqrt{n}} = \frac{CQM}{\sqrt{n}}\rho(\bw_1,\bw_2).
\end{align}
It means that $(Z_{\bw}^M)_{\bw\in T}$ is a sub-gaussian process.

According to \eqref{eqn: covering-number}, the Dudley integral of $(T_Q,\rho)$ satisfies
\begin{align}\label{eqn: dudley-integral}
\notag J(0)&=\int_0^{\diam(T_Q)}\sqrt{\log N(T_Q,\rho,\varepsilon)} \mathrm{d}\varepsilon \leq \int_0^{2Q}\sqrt{d\log\left(1+\frac{2Q}{\varepsilon}\right)}\mathrm{d}\varepsilon\\
&= 2Q\sqrt{d}\int_1^{\infty}\frac{\sqrt{\log(1+s)}}{s^2} \mathrm{d}s\leq C Q\sqrt{d}.
\end{align}
By Theorem \ref{thm: uniform-bound} and \eqref{eqn: dudley-integral}, with probability $1-2e^{-u^2}$, we have 
\begin{align}
\sup_{\bw \in T}|Z^M_{\bw}|&\lesssim \frac{QM}{\sqrt{n}}\left(J(0)+u\diam(T_Q)\right)\leq \frac{Q^2M}{\sqrt{n}}(\sqrt{d}+u).
\end{align}
Let the failure probability $2e^{-u^2}=\delta$, and we complete the proof. 
\end{proof}

\begin{proposition}\label{pro: bounding-risk-difference}
For any $\delta\in (0,1)$, with probability $1-\delta$, we have 
\[
    \sup_{\|\bw-\bw^*\|\leq Q}|\hat{\cR}_n(\bw)-\cR(\bw)|\lesssim \frac{d(Q+1)^2\sqrt{\log n}}{\sqrt{n}}\max\left\{\sqrt{\log(4/\delta)}, \frac{\log(4/\delta)}{\sqrt{n}}\right\}.
\]
\end{proposition}
\begin{proof}
Combining Lemma \ref{lemma: trucated-error} and \ref{lemma: trucated-empirical-process}, we have, with probability $1-\delta_1-\delta_2$, that 
\[
    \sup_{\bw\in T_Q}|Z_{\bw}|\lesssim  (Q+1)^2\left(d\max\left\{\sqrt{\frac{\log(2/\delta_1)}{n}}, \frac{\log(2/\delta_1)}{n}\right\} + e^{-C_2 M^2} + \frac{M}{\sqrt{n}}(\sqrt{d}+\sqrt{\log(2/\delta_2)}) \right).
\]
Taking $M=\sqrt{\frac{\log n}{2C_2}}, \delta_1=\delta_2=\delta/2$, we have 
\[
    \sup_{\bw\in T_Q}|Z_{\bw}|\lesssim \frac{d(Q+1)^2\sqrt{\log n}}{\sqrt{n}}\max\left\{\sqrt{\log(4/\delta)}, \frac{\log(4/\delta)}{\sqrt{n}}\right\}.
\]
Noting that $Z_{\bw}=\hat{\cR}(\bw)-\cR(\bw)$, we complete the proof.
\end{proof}
 
\subsection{Bounding the difference between gradients}
\label{sec: bounding-diff-grad}

In this subsection, we let $T_Q=B_Q(\bw^*)\times \SS^{d-1}$ and $\rho(\bt_1,\bt_2)=\|\bw_1-\bw_2\|+\|\bu_1-\bu_2\|$ for $\bt_1=(\bw_1,\bu_1),\bt_2=(\bw_2,\bu_2)\in T_Q$.
Let 
$
Y_{\bt}(\bx):=(\sigma(\bw^T\bx)-\sigma({\bw^*}^T\bx))\sigma'(\bw^T\bx)\bu^T\bx.
$ 
Consider the empirical process $(O_{\bt})_{\bt\in T_Q}$:
\begin{align}
O_{\bt} = \langle\bu,\nabla \hat{\cR}_n(\bw)-\nabla \cR(\bw)\rangle=\frac{1}{n}\sum_{i=1}^n Y_{\bt}(\bX_i) - \EE[Y_{\bt}(\bX)].
\end{align}
For any $M>0$, make the following decomposition
\begin{equation}\label{eqn: gradient-decomposition}
\sup_{\bt\in T_Q}|Q_{\bt}|\leq \sup_{\bt\in T_Q}|Q_{\bt}-Q_{\bt}^M| + \sup_{\bt\in T_Q}|Q_{\bt}^M|,
\end{equation}
where $Q^M_{\bt}$ is the truncated empirical process defined by
\begin{equation}
O_{\bt}^M := \frac{1}{n}\sum_{i=1}^n Y_{\bt}(\bX_i^M) - \EE[Y_{\bt}(\bX^M)].
\end{equation}
We then estimate the two terms on the right hand slide of \eqref{eqn: gradient-decomposition}, separately.

\begin{lemma}\label{lemma: truncate-gradient}
Assume $M\geq 1$.
For any $\bx_1,\bx_2\in\RR^d$ and $\bt_1,\bt_2\in T_Q$, we have 
\begin{align}
|Y_{\bt}(\bx_1)-Y_{\bt}(\bx_2)|&\lesssim (Q+1)^2\max_{i=1,2}(\|\bx_i\|^2+\|\bx_i\|)\|\bx_1-\bx_2\|\\
|Y_{\bt_1}(\bx^M) - Y_{\bt_2}(\bx^M)| &\lesssim M(1+QM) (|(\bw_1-\bw_2)^T\bx|+|(\bu_1-\bu_2)^T\bx|).
\end{align}
\end{lemma}

\begin{proof}
First,
\begin{align*}
\|\nabla_{\bx}Y_{\bt}(\bx)\| &= \|(\sigma'(\bw^T\bx)\bw-\sigma'({\bw^*}^T\bx)\bw^*)\sigma'(\bw^T\bx)\bu^T\bx  \\
&\quad \qquad + (\sigma(\bw^T\bx)-\sigma({\bw^*}^T\bx))(\sigma''(\bw^T\bx)\bu^T\bx  \bw + \sigma'(\bw^T\bx)\bu)\|\\
&\leq (\|\bw\|+\|\bw^*\|)\|\bx\| + \|\bw-\bw^*\|\|\bx\|(\|\bw\|\|\bx\|+1)\\
&\leq 2(Q+1)\|\bx\| + Q(Q+1)\|\bx\|^2.
\end{align*}
Following the mean value theorem, we have
\[
|Y_{\bt}(\bx_1)-Y_{\bt}(\bx_2)|\leq 2(Q+1)^2\max_{i=1,2}(\|\bx_i\|^2+\|\bx_i\|)\|\bx_1-\bx_2\|.
\]
Second,
\begin{align}
\nabla_{\bt}Y_{\bt}(\bx) = \begin{pmatrix}
\left(\sigma'(\bw^T\bx)^2 + (\sigma(\bw^T\bx)-\sigma({\bw^*}^T\bx)\sigma''(\bw^T\bx)\right)\bu^T\bx\bx \\
(\sigma(\bw^T\bx)-\sigma({\bw^*}^T\bx))\sigma'(\bw^T\bx)\bx
\end{pmatrix} =: \begin{pmatrix}
v_1(\bt,\bx)\bx\\
v_2(\bt,\bx)\bx
\end{pmatrix}.
\end{align}
For $\|\bx\|\leq M$, it is easy to verify that 
\[
    |v_1(\bt,\bx)|\leq M(1+QM),\qquad |v_2(\bt,\bx)|\leq QM.
\]
By the mean value theorem, there exists $\bt'$ such that 
\begin{align}
\notag |Y_{\bt_1}(\bx) - Y_{\bt_2}(\bx)| &= |\nabla_{\bt} Y_{\bt'}(\bx) (\bt_1-\bt_2)|= |v_1(\bt',\bx)(\bw_1-\bw_2)^T\bx + v_2(\bt',\bx) (\bu_1-\bu_2)^T\bx|\\
&\lesssim M(1+QM) (|(\bw_1-\bw_2)^T\bx|+|(\bu_1-\bu_2)^T\bx|).
\end{align}
\end{proof}

We then estimate the first term on the right hand side of \eqref{eqn: gradient-decomposition}.

\begin{lemma}\label{lemma: trucated-error-gradient}
There exists $C_1,C_2,C_3,C_4>0$ such that for $M>C_1 d$, with probability $1-nC_2e^{-C_3 M^2}$, we have 
\[
\sup_{\bt \in T}|Q_{\bt}-Q_{\bt}^M|\lesssim (Q+1)^2 e^{-C_4 M^2}.
\]

\end{lemma}

\begin{proof}
Using Lemma \ref{lemma: truncate-gradient} and the fact $\|\bX_i^M\|\leq \|\bX_i\|$, we have
\begin{align}
\notag |O_{\bt} - O_{\bt}^M| &\leq  \frac{1}{n}\sum_{i=1}^n |Y_{\bt}(\bX_i) - Y_{\bt}(\bX_i^M)| + \EE[|Y_{\bt}(\bX)-Y_{\bt}(\bX^M)|] \\
\notag &\lesssim \frac{(Q+1)^2}{n}\sum_{i=1}^n \|\bX_i\|(1+\|\bX_i\|)\|\bX_i-\bX_i^M\| + (Q+1)^2 \EE[\|\bX\|(1+\|\bX\|)\|\bX-\bX^M\|]\\
&\lesssim \frac{(Q+1)^2}{n}\sum_{i=1}^n V_i^M + (Q+1)^2 \EE[V^M],
\end{align}
where we let  
$$
V^M = \|\bX\|(\|\bX\|+1)\|\bX-\bX^M\|=(1+\|\bX\|)\|\bX\|^2(1-\min(1,M/\|\bX\|)).
$$
Note that for any $i\in [n]$, 
$
\PP\{V_i^M>0\}=\PP\{\|\bX\|> M\}\leq C_1 e^{-C_2 M^2}$ for $M\geq C_3 d$. 
Taking the union bound, we have 
\begin{equation}\label{eqn: truncate-empirical-gradient}
\PP\{\sum_{i=1}^n V_i^M=0\}= 1- \PP\{\sum_{i=1}^n V_i^M>0\}\geq 1- \sum_i \PP\{V_i^M>0\}\geq 1-nC_1e^{-C_2 M^2}.
\end{equation}

Similar to \eqref{eqn: expect-truncation}, we can obtain that 
\begin{align}\label{eqn: truncate-expact-gradient}
\EE[V^M] &\leq  C_1 e^{-C_2 M^2},
\end{align}
for $M\geq C_4 d$ with $C_4$ large enough. 

Combining \eqref{eqn: truncate-empirical-gradient} and \eqref{eqn: truncate-expact-gradient} completes the proof. 
\end{proof}

Before proceeding to the estimate of the second term on the right hand side of \eqref{eqn: gradient-decomposition}, we first bound the Dudley integral of the metric space. 

\begin{lemma}\label{lemma: dudley-integral-gradient}
The Dudley integral of $(T_Q,\rho)$ satisfies 
$
    J(0)\lesssim  \sqrt{d}(Q+1).
$
\end{lemma}

\begin{proof}
Note that 
\[
N(T_Q,\rho,\varepsilon)\leq N(B_Q(\bw^*),\|\cdot\|,\frac{\varepsilon}{2}) N(\SS^{d-1},\|\cdot\|,\frac{\varepsilon}{2})\leq \left(\frac{4Q}{\varepsilon}+1\right)^d \left(\frac 2 \varepsilon\right)^d.
\]
Moreover, $\diam(T)\leq 2Q+2$.
Hence, the Dudley integral is given by 
\begin{align}
\notag J(0) &= \int_0^{2Q+2}\sqrt{\log N(T,\rho,\varepsilon)} \mathrm{d}\varepsilon\\
\notag &\leq \sqrt{d}\int_0^{2Q+2}\sqrt{\log(1+\frac{4Q}{\varepsilon})+\log(2/\varepsilon)} \mathrm{d}\varepsilon\\
\notag &\leq \sqrt{d}\int_0^{2Q+2}\sqrt{\log(1+\frac{4Q}{\varepsilon})}\mathrm{d}\varepsilon+\sqrt{d}\int_0^{2Q+2}\sqrt{\log(2/\varepsilon)} \mathrm{d}\varepsilon\\
\notag &\leq \sqrt{d}4Q\int_{\frac{2Q}{Q+1}}^{\infty}\frac{\sqrt{\log(1+ s)}}{s^2} \mathrm{d}s + 2\sqrt{d}\int_{\frac{1}{Q+1}}^{\infty}\frac{\sqrt{\log s}}{s^2}\mathrm{d}s\\
&\lesssim \sqrt{d}(Q+1).
\end{align}
\end{proof}

\begin{lemma}\label{lemma: trucated-empirical-process-gradient}
For any $u>0$, with probability $1-2e^{-u^2}$, we have 
\[
\sup_{\bw \in T}|Z^M_{\bw}|\leq \frac{(Q+1)^2M^2}{\sqrt{n}}(\sqrt{d}+u).
\]
\end{lemma}
\begin{proof}

By Lemma \ref{lemma: pro-orlic-norm} and \ref{lemma: truncate-gradient} , we have 
\begin{align}
\notag \|Y_{\bt_1}(\bX^M)-Y_{\bt_2}(\bX^M)\|_{\psi_2}&\lesssim M(1+QM) \||(\bw_1-\bw_2)^T\bX^M|+|(\bu_1-\bu_2)^T\bX^M|\|_{\psi_2} \\
\notag &\lesssim M(1+QM)  (\|\bw_1-\bw_2\|+\|\bu_1-\bu_2\|)\\
&= M(1+QM)\rho(\bt_1,\bt_2).
\end{align}
where the second inequality is due to Eq. \eqref{eqn: 11}. By Proposition \ref{pro: sum-sub-gaussians}, we have 
\begin{align}
\notag \|O_{\bt_1}^M - O_{\bt_2}^M\|_{\psi_2} &= \big\|\frac{1}{n}\sum_{i=1}^n (Y_{\bt_1}(\bX_i^M) -Y_{\bt_2}(\bX_i^M) - \EE[Y_{\bt_1}(\bX^M)] - \EE[Y_{\bt_2}(\bX^M)])\big\|_{\psi_2}\\
\notag &\lesssim \frac{1}{n}\sqrt{\sum_{i=1}^n \|Y_{\bt_1}(\bX_i^M) -Y_{\bt_2}(\bX_i^M) - \EE[Y_{\bt_1}(\bX^M)] - \EE[Y_{\bt_2}(\bX^M)]\|^2_{\psi_2}}\\
\notag &\lesssim \frac{1}{n} \sqrt{\sum_{i=1}^n \|Y_{\bt_1}(\bX_i^M) -Y_{\bt_2}(\bX_i^M)\|_{\psi_2}^2}\\
&\lesssim \frac{M(1+QM)}{\sqrt{n}}\rho(\bt_1,\bt_2).
\end{align}
It means that $(O_{\bt}^M)_{\bt\in T_Q}$ is a sub-gaussian process. 
Moreover, $\diam(T_Q)=2(Q+1)$.

By Theorem \ref{thm: uniform-bound} and Lemma \ref{lemma: dudley-integral-gradient}, with probability $1-2e^{-u^2}$, we have 
\begin{align}
\sup_{\bw \in T}|Z^M_{\bw}|&\lesssim \frac{M(1+QM)}{\sqrt{n}}\left(J(0)+u\diam(T_Q)\right)\leq \frac{(Q+1)^2M^2}{\sqrt{n}}(\sqrt{d}+u).
\end{align}
\end{proof}

\begin{proposition}\label{pro: gradient-difference}
Assume $n\geq 3$. For any $\delta\in (0,1)$, with probability $1-\delta$, we have 
\[
    \sup_{\|\bw-\bw^*\|\leq Q}|\nabla \hat{\cR}_n(\bw)-\nabla \cR(\bw)|\lesssim \frac{\sqrt{d}\log^{3/2}(n/\delta)}{\sqrt{n}}(Q+1)^2.
\]
\end{proposition}
\begin{proof}
Combining Lemma \ref{lemma: trucated-error-gradient} and \ref{lemma: trucated-empirical-process-gradient},  with probability $(1-nC_2 e^{-C_3 M^2})(1-2e^{-u^2})$, we have  for $M\geq C_1 d$,
\[
    \sup_{\bw\in T_Q}|Q_{\bt}|\lesssim  (Q+1)^2\left(e^{-C_4 M^2}+\frac{M^2}{\sqrt{n}}\big(\sqrt{d}+u\big) \right).
\]
Taking $M^2=\frac{\log(2nC_2/\delta)}{\min(C_3,C_4)}$ and $u=\sqrt{\log(4/\delta)}$, we have with probability $(1-\delta/2)^2\geq 1-\delta$ that 
\[
    \sup_{\bw\in T_Q}|Q_{\bt}|\leq C(Q+1)^2\left(\frac{\delta}{2C_2n}+\frac{\log(2C_2 n/\delta)}{\min(C_3,C_4)\sqrt{n}}(\sqrt{d}+\sqrt{\log(4/\delta)})\right).
\]
Assuming that $n\geq 3$, the above inequality can be simplified as follows 
\[
\sup_{\bw\in T_Q}|Q_{\bt}|\lesssim \frac{\sqrt{d}\log^{3/2}(n/\delta)}{\sqrt{n}}(Q+1)^2.
\]
Noting that
$$
\sup_{\bw\in B_Q(\bw^*)}\|\nabla \hat{\cR}_n(\bw)-\nabla \cR(\bw)\| = \sup_{\bw \in B_Q(\bw^*)}\sup_{\bu\in\SS^{d-1}} \bu^T(\nabla \hat{\cR}_n(\bw)-\nabla \cR(\bw)) = \sup_{\bt\in T_Q} O_{\bt}.
$$ 
we complete the proof.
\end{proof}

\end{document}